\documentclass{article}
\usepackage{amsmath,amssymb,amsthm}
\usepackage{siunitx}
\usepackage{graphicx}
\usepackage{etoolbox}
\usepackage{caption}
\usepackage{subcaption}
\usepackage{float}
\usepackage{booktabs}
\usepackage{multirow}
\usepackage{threeparttable}
\usepackage{tabularx}
\usepackage{colortbl}
\usepackage{xcolor}
\usepackage{textcomp}
\usepackage{afterpage}
\usepackage{gensymb}
\usepackage{wrapfig}
\usepackage{enumitem}
\usepackage{array}
\newcolumntype{L}[1]{>{\raggedright\arraybackslash}m{#1}}
\newcolumntype{C}[1]{>{\centering\arraybackslash}m{#1}}
\newcolumntype{R}[1]{>{\raggedleft\arraybackslash}m{#1}}

\usepackage{algorithm}
\usepackage[noend]{algpseudocode}  

\algrenewcommand\algorithmicrequire{\textbf{Require:}}
\algrenewcommand\algorithmicensure{\textbf{Ensure:}}

\newcommand{\boldtheorem}[1]{\textbf{#1}}
\newtheorem{theorem}{\boldtheorem{Theorem}}
\newtheorem{lemma}{\boldtheorem{Lemma}}

\usepackage[hidelinks]{hyperref}  
\hypersetup{
    colorlinks=true,
    linkcolor=black,
    citecolor=blue,
    urlcolor=blue,
    pdfborder={0 0 0}  
}

\title{A Multi-Scale Graph Neural Process with Cross-Drug Co-Attention for Drug-Drug Interactions Prediction}
\author{
{ Zimo Yan\textsuperscript{\rm 1}, Zheng Xie\textsuperscript{\rm 1}\thanks{*Corresponding author: Zheng Xie (xiezheng81@nudt.edu.cn)}, Jie Zhang\textsuperscript{\rm 1}, Song Yiping\textsuperscript{\rm 1}, Li Hao\textsuperscript{\rm 1}. }
\vspace{1.6mm}\\
\fontsize{10}{10}\selectfont\itshape
\textsuperscript{\rm 1}National University of Defense Technology, Changsha, China.
\\\{yanzimo20, xiezheng81, zhangjie, songyiping, lihao22}
\fontsize{9}{9}\selectfont\ttfamily\upshape
\begin{document}

\maketitle

\begin{abstract}
Accurate prediction of drug-drug interactions (DDI) is crucial for medication safety and effective drug development. However, existing methods often struggle to capture structural information across different scales—from local functional groups to global molecular topology—and typically lack mechanisms to quantify prediction confidence. To address these limitations, we propose MPNP-DDI, a novel Multi-scale Graph Neural Process framework. The core of MPNP-DDI is a unique message-passing scheme that, by being iteratively applied, learns a hierarchy of graph representations at multiple scales. Crucially, a cross-drug co-attention mechanism then dynamically fuses these multi-scale representations to generate context-aware embeddings for interacting drug pairs, while an integrated neural process module provides principled uncertainty estimation. Extensive experiments demonstrate that MPNP-DDI significantly outperforms state-of-the-art baselines on benchmark datasets. By providing accurate, generalizable, and uncertainty-aware predictions built upon multi-scale structural features, MPNP-DDI represents a powerful computational tool for pharmacovigilance, polypharmacy risk assessment, and precision medicine.
\end{abstract}


\section{Introduction}
The concurrent use of multiple medications, known as polypharmacy, is increasingly common, elevating the risk of adverse drug events stemming from unforeseen drug-drug interactions \cite{hines2021trends,gottlieb2020polypharmacy}. To mitigate these risks, predicting these interactions serves as a cornerstone of pharmacovigilance and clinical decision support. However, the challenge is immense, as the number of potential DDIs grows combinatorially with the number of available drugs, making exhaustive experimental screening infeasible \cite{ryu2018deep}. This reality underscores the critical importance of developing accurate and scalable computational models to forecast DDI risks preemptively.

Initial computational approaches for DDI prediction relied heavily on literature mining to extract known interactions from biomedical texts \cite{tari2010nalgene}, or similarity-based methods that assume drugs with similar properties (e.g., chemical structure, target proteins) are likely to share similar interaction profiles \cite{gottlieb2012predicting}. In recent years, Graph Neural Networks (GNNs) have emerged as the state-of-the-art for learning from molecular data \cite{gilmer2017neural}. When applying GNNs to the DDI problem, which inherently involves a pair of drugs, the dual-GNN architecture has become a common paradigm. In this setup, two separate GNNs process the paired drugs independently, and their final embeddings are concatenated for prediction \cite{feng2020n,deac2023drug}.

Beyond these foundational models, emerging strategies are tackling the DDI prediction problem with greater complexity. Knowledge graph-based methods embed drugs within a larger biomedical network, incorporating heterogeneous information such as proteins, diseases, and side effects to enrich drug representations \cite{zitnik2018modeling,lin2020kgnn}. In parallel, multi-modal approaches aim to fuse diverse data sources, such as molecular structures and textual descriptions, to create more comprehensive drug profiles \cite{deng2020multi}. Other advanced models have begun to incorporate co-attention mechanisms to model substructure-level interactions \cite{Ma2023dual}.

Despite these advancements, a fundamental limitation persists, as many models are built upon standard GNNs that operate at a single, fixed analytical scale. This prevents them from simultaneously capturing fine-grained local substructures and global molecular topology, often seeing the trees but not the forest \cite{alon2021on}, while also lacking a mechanism to dynamically focus on the most salient chemical motifs \cite{velivckovic2018graph}. This architectural scale-insensitivity leads to a more profound conceptual flaw: the generation of static, context-agnostic drug representations. In prevalent dual-GNN pipelines, the representation of Drug A is computed in an "information silo," entirely independent of its partner, Drug B \cite{feng2020n,deac2023drug}. This approach is fundamentally misaligned with chemical reality, where a drug's interactive potential is dynamic and context-dependent. A truly effective model must therefore first perceive features across multiple scales to then generate a dynamic, context-aware representation that reflects how these features are expressed in the presence of a specific partner \cite{Ma2023dual}.

\textbf{Motivation}:
This raises the following question: \textit{Can we design a DDI prediction model that moves beyond static, single-scale feature extraction and instead learns dynamic, context-aware representations from a rich hierarchy of multi-scale features, while also quantifying the prediction's reliability?} To address this challenge, we introduce the  Multi-scale Graph Neural Process for DDI (MPNP-DDI), 
as illustrated in Figure~\ref{fig:model_architecture}, our framework addresses the single-scale issue by using stacked GNP blocks to extract a hierarchy of features, and breaks the "information silo" by dynamically fusing these features with a cross-drug co-attention mechanism. This interactive process yields a context-aware representation for predicting both the DDI event and the model's uncertainty.
\begin{figure*}[t!]
    \centering
    \includegraphics[width=\textwidth]{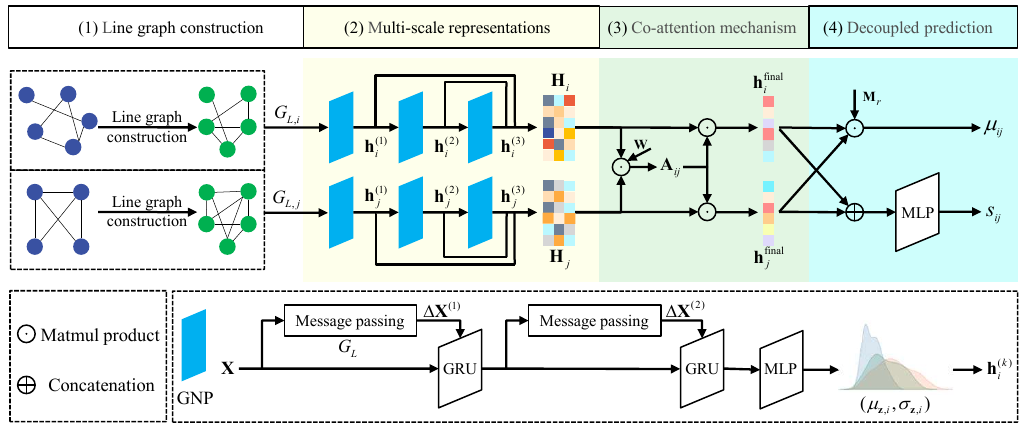} 
    \caption{The architecture of MPNP-DDI. Stacked GNP blocks generate multi-scale representations for each drug, which are then fused by a cross-drug co-attention mechanism to enable context-aware, probabilistic DDI prediction.}
    \label{fig:model_architecture}
\end{figure*}

\textbf{Primary contributions.} The primary contributions of this work are as follows:
\begin{enumerate}
    \item We propose a multi-scale framework for DDI prediction, first extracts a hierarchy of representations using stacked GNP blocks and then learns context-aware embeddings by dynamically fusing these features via a cross-drug co-attention mechanism.
    \item We integrate an uncertainty estimation module to quantify the model's confidence for each prediction, enhancing its reliability and applicability in clinical settings.
    \item We demonstrate the effectiveness of our multi-scale, interactive approach through extensive experiments, achieving highly competitive performance on key DDI prediction benchmarks.
\end{enumerate}

\section{Literature Review}
\label{sec:literature_review}

This section provides a review of existing computational methods for DDI prediction, organized into four key areas that reflect the evolution of the field. We begin with traditional and similarity-based approaches, then delve into the now-dominant Graph Neural Network-based models. Subsequently, we discuss methods that integrate richer external information, and finally, we highlight the emerging trend towards more advanced graph models and the critical need for uncertainty quantification.

\textbf{Foundational Computational Methods.}
The foundational computational approaches for DDI prediction were primarily built upon two pillars: literature mining and similarity-based inference. Literature mining employs Natural Language Processing (NLP) techniques to automatically extract known DDI pairs from vast biomedical text corpora, such as scientific articles and clinical reports \cite{tari2010nalgene, percha2012discovery}. While effective for cataloging documented interactions, this strategy is inherently unable to predict novel, previously unobserved DDIs.

In parallel, similarity-based methods operate on the widely held pharmacological principle that similar drugs are likely to exhibit similar interaction behaviors \cite{gottlieb2012predicting}. These models represent drugs using a variety of feature vectors—including 2D/3D chemical fingerprints, target protein profiles, and side effect signatures \cite{vilar2012drug, cheng2013machine}—and infer the interaction probability for a new pair based on its similarity to known interacting pairs. The primary limitation of this approach is its heavy reliance on hand-crafted features and the "similarity assumption," which may not always hold true, as structurally dissimilar drugs can sometimes trigger similar pathways.

\textbf{Knowledge Graph and Multi-Modal Approaches.}
To move beyond molecular structure alone, another line of research has focused on integrating richer, heterogeneous information. Knowledge Graph (KG) based methods have gained significant traction by embedding drugs into large-scale biomedical networks \cite{zitnik2018modeling}. These KGs connect drugs to other entities like proteins, genes, diseases, and side effects, allowing models to learn drug representations enriched with relational context from the broader biological landscape \cite{lin2020kgnn, yao2021tri}.

Concurrently, multi-modal approaches aim to create a more holistic drug profile by fusing data from diverse sources. These models combine molecular structures with other data modalities, such as textual descriptions from drug labels, gene expression profiles, and clinical data, to predict DDIs \cite{deng2020multi, zhang2023deep}. While powerful, both KG and multi-modal methods are highly dependent on the availability, quality, and completeness of external databases, which can be noisy, sparse, or suffer from inherent biases. This dependency can limit their generalization and highlights the value of methods that can extract maximal information from the most fundamental data source: the molecular structure itself.

\textbf{Graph Neural Network-Based DDI Prediction.}
The advent of Graph Neural Networks (GNNs) marked a paradigm shift, enabling models to learn representations directly from the raw molecular graph structure \cite{gilmer2017neural}. In the context of DDI prediction, this led to the prevalence of the \textbf{dual-GNN architecture}. In this paradigm, two GNNs (or a single shared-weight GNN) independently process the two drug molecules to generate fixed-size embeddings, which are then concatenated and fed into a classifier to predict the interaction type \cite{feng2020n, deac2023drug}. 

Building on this foundation, subsequent works have sought to improve the feature extraction process. For instance, SSI-DDI \cite{yu2021ssi} utilized graph attention layers to focus on important substructures, while GMPNN-CS \cite{liu2022gmpnn} introduced a gated message passing mechanism to capture substructures of varying sizes. More recently, DGNN-DDI \cite{Ma2023dual} incorporated a co-attention mechanism to weigh the importance of substructure interactions. However, a common limitation persists: these methods still largely operate on static representations. The co-attention is often applied as a late-stage fusion step on pre-computed, context-agnostic features, failing to fully resolve the "information silo" problem and model the dynamic nature of drug interactions at the feature learning level.

\textbf{Advanced Graph Models and Uncertainty Quantification.}
The limitations of standard GNNs have motivated the exploration of more advanced graph learning frameworks. Graph Neural Processes (GNPs) represent a promising frontier \cite{garnelo2018conditional}. Unlike conventional GNNs, GNPs learn a distribution over functions on graphs. This probabilistic approach yields two significant benefits for DDI prediction. First, it makes GNPs inherently suitable for few-shot generalization; their underlying meta-learning structure is particularly promising for adapting to new drugs or rare DDI types, which are common challenges in real-world datasets \cite{kim2019attentive}.

Second, and most critically, GNPs provide a principled mechanism for uncertainty quantification. The vast majority of existing DDI models provide deterministic outputs, offering no indication of prediction reliability. In safety-critical domains like medicine, this is a significant shortcoming. A model that can express its own uncertainty would be transformative: low-confidence predictions could alert clinicians to the need for cross-validation, while high-uncertainty predictions on novel pairs could guide experimental research by pinpointing the boundaries of current knowledge. Despite this clear potential, the application of GNPs to DDI prediction remains largely unexplored, highlighting a critical gap that this work aims to address.

\section{Preliminaries}
\label{sec:preliminaries}
This section introduces the fundamental notations for representing drugs as graphs and formally defines the task of probabilistic DDI prediction.
\subsection{Notations}
We model each drug as a molecular graph and formulate the DDI prediction task as learning a mapping from a pair of graphs to a probabilistic output.

A single drug molecule is represented as a graph $G=(V, E)$, where $V$ is the set of atoms (nodes) and $E$ is the set of chemical bonds (edges). Each node $v \in V$ is associated with an initial feature vector $\mathbf{x}_v \in \mathbb{R}^{d_v}$, encoding atomic properties (e.g., atom type, degree, formal charge). Similarly, each edge $e_{uv} \in E$ connecting nodes $u$ and $v$ has a feature vector $\mathbf{e}_{uv} \in \mathbb{R}^{d_e}$, representing bond properties (e.g., bond type, conjugation).

To facilitate the edge-to-edge message passing scheme central to our model, we also construct the line graph $G_L=(V_L, E_L)$ for each molecular graph $G$. In the line graph, each node $v' \in V_L$ corresponds to an edge in the original graph $G$. An edge $e'_{v'u'} \in E_L$ exists between two nodes $v'$ and $u'$ if their corresponding bonds in $G$ share a common atom.

In a multi-scale setting, we generate a set of hierarchical representations for each drug. For a drug $d_i$, its multi-scale representation is denoted as $\mathbf{H}_i = \{\mathbf{h}_{i}^{(k)}\}_{k=1}^{K}$, where $\mathbf{h}_{i}^{(k)} \in \mathbb{R}^{d_h}$ is the graph-level embedding at scale $k$, and $K$ is the total number of scales (i.e., the number of GNP blocks).

\begin{table}[h!]
\centering
\caption{Summary of key notations.}
\label{tab:notations}
\begingroup 
\footnotesize 
\setlength{\tabcolsep}{4pt} 

\begin{tabular}{ll}
\toprule
\textbf{Notation} & \textbf{Description} \\
\midrule
$G_i, G_j$ & Molecular graphs for a drug pair. \\
$\mathbf{x}_v \in \mathbb{R}^{d_v}$ & Atom (node) feature vector. \\
$\mathbf{e}_{uv} \in \mathbb{R}^{d_e}$ & Bond (edge) feature vector. \\
$G_{L,i}$ & Line graph derived from $G_i$. \\
\addlinespace 
$\mathbf{h}_{i}^{(k)}$ & Graph embedding of drug $d_i$ at scale $k$. \\
$\mathbf{H}_i$ & Set of multi-scale embeddings for a drug. \\
\addlinespace
$\mathcal{T}_{\text{train}}$ & Training set of labeled drug pairs. \\
$y_{ij}$ & Binary interaction label (1=yes, 0=no). \\
$f_\theta(\cdot, \cdot)$ & The DDI prediction model with parameters $\theta$. \\
\addlinespace
$(\mu_{ij}, s_{ij})$ & Model output: interaction logit and log-variance. \\
$\sigma_{ij}^2$ & Estimated variance, derived from $s_{ij}$. \\
\addlinespace
$\mathcal{L}_{\text{MPNP}}$ & The total loss function. \\
\begin{tabular}{@{}l@{}}$\mathcal{L}_{\text{pred}}, \mathcal{L}_{\text{unc}},$ \\ $\mathcal{L}_{\text{kl}}$\end{tabular} 
& Prediction, uncertainty, and KL loss components. \\
$\lambda_{\text{unc}}, \lambda_{\text{kl}}$ & Weights for the loss components. \\
\bottomrule
\end{tabular}
\endgroup 
\end{table}

\subsection{Problem Statement}
We address the task of DDI prediction as a probabilistic binary classification problem on graph pairs. The model must not only predict the likelihood of an interaction but also quantify the confidence in its prediction.

Formally, given a training set of drug graph pairs with interaction labels:
\[
\mathcal{T}_{\text{train}} = \{(G_i, G_j, y_{ij})\}_{i,j},
\]
where $G_i$ and $G_j$ are the molecular graphs for a pair of drugs and $y_{ij} \in \{0, 1\}$ is the ground-truth binary label, the objective is to learn a probabilistic predictive function $f_\theta$. This function maps a pair of graphs to a predictive distribution, which we parameterize by its logit and log-variance:

\begin{equation}
f_\theta: (G_i, G_j) \mapsto (\mu_{ij}, s_{ij})
\label{eq:mapping_function}
\end{equation}
Here, $\mu_{ij} \in \mathbb{R}$ is the predicted logit for the interaction, and $s_{ij} \in \mathbb{R}$ is the predicted log-variance. The variance $\sigma_{ij}^2 = \exp(s_{ij})$ serves as the model's uncertainty for the prediction.

The model is trained by minimizing a composite loss function, $\mathcal{L}_{\text{MPNP}}$, which balances predictive accuracy with calibrated uncertainty estimation and regularization. This loss function is composed of a prediction term, an uncertainty regularization term, and a KL divergence term:
\begin{equation}
\mathcal{L}_{\text{MPNP}} = \mathcal{L}_{\text{pred}} + \lambda_{\text{unc}} \mathcal{L}_{\text{unc}} + \lambda_{\text{kl}} \mathcal{L}_{\text{kl}}
\label{eq:total_loss}
\end{equation}
where $\lambda_{\text{unc}}$ and $\lambda_{\text{kl}}$ are hyperparameters that control the influence of their respective loss components.

The primary component is the prediction loss, $\mathcal{L}_{\text{pred}}$, which drives the model's accuracy. It uses the standard binary cross-entropy with logits formulation:
\begin{equation}
\mathcal{L}_{\text{pred}} = \mathbb{E}_{(G_i, G_j, y_{ij})} \left[ -y_{ij}\mu_{ij} + \log(1 + \exp(\mu_{ij})) \right]
\label{eq:pred_loss}
\end{equation}

The second component, the uncertainty loss $\mathcal{L}_{\text{unc}}$, regularizes the model's confidence. It is designed to penalize predictions that are both incorrect and highly confident. This is achieved by scaling the squared prediction error by the inverse of the model's estimated variance, and adding a regularization term on the variance itself:
\begin{equation}
\mathcal{L}_{\text{unc}} = \mathbb{E}_{(G_i, G_j, y_{ij})} \left[ (\text{sigmoid}(\mu_{ij}) - y_{ij})^2 \exp(-s_{ij}) + s_{ij} \right]
\label{eq:unc_loss}
\end{equation}
where $\text{sigmoid}(\cdot)$ is the sigmoid function. The first term, $(\text{sigmoid}(\mu_{ij}) - y_{ij})^2 / \sigma_{ij}^2$, forces the model to learn a high variance (high uncertainty) for incorrect predictions to avoid a large loss. The second term, $s_{ij} = \log(\sigma_{ij}^2)$, acts as a regularizer to prevent the model from trivially predicting infinite variance for all samples.

The third component, $\mathcal{L}_{\text{kl}}$, represents a Kullback-Leibler (KL) divergence term that regularizes the latent space of the graph embeddings, promoting a more robust and generalizable representation.

At inference time, for a given drug pair $(G_i, G_j)$, the model outputs the probability of interaction $p_{ij} = \text{sigmoid}(\mu_{ij})$ and the uncertainty score $\sigma_{ij}^2 = \exp(s_{ij})$. A high $\sigma_{ij}^2$ indicates that the prediction is less reliable and should be treated with caution, a critical feature for clinical applications.

\section{Method}
\label{sec:method}

We develop MPNP-DDI (Multi-scale Graph Neural Process for DDI), a novel framework designed for probabilistic drug-drug interactions prediction directly from molecular graphs. Unlike traditional models that generate static, context-agnostic embeddings, MPNP-DDI is architected to produce dynamic, context-aware representations that reflect the specific chemical environment of a given drug pair. The core of our framework rests on four key innovations: (1) a novel \textbf{edge-to-edge message passing scheme} on the line graph to capture high-order topological information; (2) a hierarchical feature extractor composed of stacked \textbf{Graph Neural Process (GNP) blocks} that leverage this scheme to learn a distribution over multi-scale representations; (3) a \textbf{cross-drug co-attention mechanism} to dynamically fuse these multi-scale features; and (4) a \textbf{decoupled prediction head} that simultaneously outputs an interaction score and its corresponding uncertainty.

By explicitly learning how the structural features of one drug are perceived in the presence of another, MPNP-DDI moves beyond simple pattern recognition towards a more nuanced, relational understanding of drug-drug interactions.

\subsection{Edge-to-Edge Message Passing for High-Order Structures}
\label{subsec:edge_message_passing}
Our first innovation is a message passing scheme that operates between chemical bonds rather than atoms to capture high-order structural motifs. To facilitate this, we first represent each drug molecule $d_i$ as a standard graph $G_i=(V_i, E_i)$, where atoms are nodes and bonds are edges, with initial features $\mathbf{x}_v$ and $\mathbf{e}_{uv}$. These raw features are projected into a unified hidden dimension $d_h$ to yield initial hidden states $\mathbf{x}_v^{(0)}$ and $\mathbf{e}_{uv}^{(0)}$:
\begin{equation}
\mathbf{x}_v^{(0)} = \text{PReLU}(\text{BatchNorm}(\phi_v(\mathbf{x}_v))), \quad \mathbf{e}_{uv}^{(0)} = \phi_e(\mathbf{e}_{uv})
\label{eq:preprocessing}
\end{equation}
The key to our scheme is the construction of the \textbf{line graph} $G_{L,i}=(V_{L,i}, E_{L,i})$. In the line graph, each node in $V_{L,i}$ corresponds to a bond in the original graph $E_i$. An edge connects two nodes in the line graph if their corresponding bonds in $G_i$ are incident to the same atom.

Upon this structure, the message passing process unfolds. For each bond $e_{uv}$, an initial message $\mathbf{m}_{uv}$ is created by combining its features with those of its incident nodes:
\begin{equation}\label{eq:msg_init}
    \mathbf{m}_{uv} = \mathbf{e}_{uv}^{(t)} + \frac{1}{2}(\mathbf{x}_u^{(t)} + \mathbf{x}_v^{(t)})
\end{equation}
These messages are then refined by aggregating information from neighboring bonds via the line graph connectivity. The updated message $\mathbf{m}'_{uv}$ is computed as:
\begin{equation}\label{eq:msg_agg}
    \mathbf{m}'_{uv} = \mathbf{m}_{uv} + \sum_{e_{jk} \in \mathcal{N}_L(e_{uv})} \mathbf{m}_{jk}
\end{equation}
where $\mathcal{N}_L(e_{uv})$ are the neighboring bonds of $e_{uv}$ in the line graph. These refined bond-level messages $\mathbf{m}'$ form the basis for updating the atom representations in subsequent stages.

\subsection{Multi-Scale Stochastic Representation Learning via GNP Blocks}
\label{subsec:gnp_blocks}
The second innovation is a stack of $K$ Graph Neural Process (GNP) blocks that generate multi-scale stochastic representations. Each block uses the refined bond messages $\mathbf{m}'$ from the previous stage to update node representations and then defines a distribution over graph-level embeddings for a specific structural "scale". This involves two main steps: a recurrent state update and a stochastic graph readout.

\textbf{Recurrent State Update.} The refined edge messages are scattered back to the nodes. The aggregated update signal for a node $v$ is $\Delta \mathbf{x}_v^{(t)} = \sum_{u \in \mathcal{N}(v)} \mathbf{m}'_{uv}$. To model the dynamics of this iterative process over $T$ steps within a block, we employ a Gated Recurrent Unit (GRU) \cite{cho2014learning}:
\begin{equation}\label{eq:gru_update}
    \mathbf{x}_v^{(t+1)} = \text{GRU}(\Delta \mathbf{x}_v^{(t)}, \mathbf{x}_v^{(t)})
\end{equation}

\textbf{Stochastic Graph Readout.} After $T$ iterations, the final node representations $\mathbf{X}^{(k)}$ are used to parameterize a latent distribution. A deterministic summary vector $\mathbf{s}_i^{(k)}$ is first obtained via attention pooling, which is then projected to produce the mean $\boldsymbol{\mu}_{\mathbf{z},i}^{(k)}$ and log-variance $\boldsymbol{\sigma}_{\mathbf{z},i}^{(k)}$ of a diagonal Gaussian posterior $q(\mathbf{z}_i^{(k)} | G_i)$. The multi-scale representation for drug $d_i$ is a set of samples $\mathbf{H}_i = \{\mathbf{h}_i^{(1)}, \dots, \mathbf{h}_i^{(K)}\}$ drawn from these distributions, regularized by a KL divergence loss $\mathcal{L}_{\text{kl}}$ against a standard normal prior.

\subsection{Dynamic Interaction Modeling with Co-Attention}
\label{subsec:co_attention}
The third innovation is a co-attention mechanism that dynamically fuses the multi-scale representations $\mathbf{H}_i$ and $\mathbf{H}_j$ for a drug pair $(d_i, d_j)$, generating context-aware embeddings. First, a cross-scale affinity matrix $\mathbf{A} \in \mathbb{R}^{K \times K}$ is computed:
\begin{equation}\label{eq:affinity}
    A_{kl} = (\mathbf{h}_i^{(k)})^\top \mathbf{W} \mathbf{h}_j^{(l)}
\end{equation}
where $\mathbf{W} \in \mathbb{R}^{d_h \times d_h}$ is a learnable matrix. From $\mathbf{A}$, we derive attention weights $\boldsymbol{\alpha}_i$ and $\boldsymbol{\alpha}_j$ for each drug by applying a softmax function to the mean-pooled columns and rows, respectively. The final context-aware embedding for each drug is a weighted sum of its multi-scale features:
\begin{equation}\label{eq:final_embedding}
    \mathbf{h}_i^{\text{final}} = \sum_{k=1}^{K} \alpha_{ik} \mathbf{h}_i^{(k)}, \quad \mathbf{h}_j^{\text{final}} = \sum_{l=1}^{K} \alpha_{jl} \mathbf{h}_j^{(l)}
\end{equation}
These fused embeddings, $\mathbf{h}_i^{\text{final}}$ and $\mathbf{h}_j^{\text{final}}$, now contain information about their interaction partner.

\subsection{Decoupled Probabilistic Prediction}
\label{subsec:prediction_head}
The final innovation is a decoupled prediction architecture operating on the context-aware embeddings to produce a probabilistic output. It consists of two parallel heads.

\textbf{Prediction Head.} This head uses the RESCAL model \cite{nickel2011three} to predict an interaction logit $\mu_{ij}$ from the final embeddings $\mathbf{h}_i^{\text{final}}$ and $\mathbf{h}_j^{\text{final}}$:
\begin{equation}
    \mu_{ij} = (\mathbf{h}_i^{\text{final}})^\top \mathbf{M}_r \mathbf{h}_j^{\text{final}}
\end{equation}
where $\mathbf{M}_r$ is a learnable, relation-specific matrix.

\textbf{Uncertainty Head.} This head, an MLP, quantifies model confidence by predicting the log-variance $s_{ij}$. Crucially, it also operates on the final context-aware embeddings to ensure the uncertainty estimate is conditioned on the specific drug pair:
\begin{equation}
    s_{ij} = \text{MLP}([\mathbf{h}_i^{\text{final}} ; \mathbf{h}_j^{\text{final}}])
\end{equation}
The training objective combines a prediction loss with an uncertainty-aware term and the KL divergence from the GNP blocks, as detailed in Algorithm~\ref{alg:MPNP-DDI_training}.

\begin{algorithm}[h!]
\caption{MPNP-DDI Training Procedure}
\label{alg:MPNP-DDI_training}
\begin{flushleft}
    \textbf{Input:} Training dataloader $\mathcal{D}_{\text{train}}$, model with parameters $\theta$, optimizer $\mathcal{O}$ \\
    \textbf{Hyperparameters:} learning rate $\eta$, uncertainty weight $\lambda_{\text{unc}}$, KL weight $\lambda_{\text{kl}}$ \\
    \textbf{Output:} Optimized model parameters $\theta^*$
    \vspace{3mm}
    
    1: \textbf{Function} ComputeForwardPass($G_i, G_j, \theta$): \\
    2: \quad $(\mathbf{H}_i, \text{KL}_i), (\mathbf{H}_j, \text{KL}_j) \leftarrow \text{GNP\_Encoder}(G_i, G_j)$ \\
    3: \quad $\mathbf{h}_i^{\text{final}}, \mathbf{h}_j^{\text{final}} \leftarrow \text{CoAttention}(\mathbf{H}_i, \mathbf{H}_j)$ \\
    4: \quad $\mu_{ij} \leftarrow \text{PredictionHead}(\mathbf{h}_i^{\text{final}}, \mathbf{h}_j^{\text{final}})$ \\
    5: \quad $s_{ij} \leftarrow \text{UncertaintyHead}(\mathbf{h}_i^{\text{final}}, \mathbf{h}_j^{\text{final}})$ \\
    6: \quad \textbf{return} $(\mu_{ij}, s_{ij}, \text{KL}_i, \text{KL}_j)$ \\
    7: \textbf{end Function}
    \vspace{3mm}

    8: Initialize model parameters $\theta$ \\
    9: \textbf{for} each training epoch \textbf{do} \\
    10: \quad \textbf{for} each batch $(G_i, G_j, y_{ij})$ in $\mathcal{D}_{\text{train}}$ \textbf{do} \\
    11: \qquad $(\mu_{ij}, s_{ij}, \text{KL}_i, \text{KL}_j) \leftarrow \text{ComputeForwardPass}(G_i, G_j, \theta)$ \\
    12: \qquad $\mathcal{L}_{\text{pred}} \leftarrow \text{BCEWithLogitsLoss}(\mu_{ij}, y_{ij})$ \\
    13: \qquad $\mathcal{L}_{\text{unc}} \leftarrow \text{mean}((\text{sigmoid}(\mu_{ij}) - y_{ij})^2 \cdot \exp(-s_{ij}) + s_{ij})$ \\
    14: \qquad $\mathcal{L}_{\text{kl}} \leftarrow \text{mean}(\text{KL}_i) + \text{mean}(\text{KL}_j)$ \\
    15: \qquad $\mathcal{L}_{\text{total}} \leftarrow \mathcal{L}_{\text{pred}} + \lambda_{\text{unc}} \mathcal{L}_{\text{unc}} + \lambda_{\text{kl}} \mathcal{L}_{\text{kl}}$ \\
    16: \qquad Compute gradient $\nabla_\theta \mathcal{L}_{\text{total}}$ \\
    17: \qquad Update parameters $\theta$ using optimizer $\mathcal{O}$ \\
    18: \quad \textbf{end for} \\
    19: \textbf{end for} \\
    20: \textbf{return} Optimized parameters $\theta^*$
\end{flushleft}
\end{algorithm}

\section{Theoretical Foundations}
\label{sec:theoretical_analysis}

In this section, we provide a theoretical analysis of our proposed MPNP-DDI framework. We first analyze the convergence properties of our composite loss function, which integrates prediction accuracy, uncertainty calibration, and latent space regularization. We then discuss the model's generalization capabilities through the lens of PAC-Bayesian theory, highlighting the role of the uncertainty component. Finally, we analyze the enhanced expressive power of our model compared to standard Message Passing Neural Networks (MPNNs).

\subsection{Convergence Analysis}

The goal of this section is to formally establish the convergence of our training procedure. We follow a constructive, bottom-up approach. First, we lay out the standard assumptions for our analysis. We then prove a series of lemmas concerning the Lipschitz continuity of our model's core components. Based on these lemmas, we establish the crucial L-smoothness property of our loss function. Finally, with this property proven, we present the main convergence guarantee for stochastic gradient descent.

\textbf{Assumptions.}
Our analysis relies on the following standard assumptions:
\begin{itemize}
    \item[\textbf{A1}] \textbf{Bounded Parameters:} All learnable weight matrices $\mathbf{W}$ in the model have a bounded spectral norm, $\|\mathbf{W}\|_2 \le C_W < \infty$.
    \item[\textbf{A2}] \textbf{Bounded Inputs:} Initial node and edge features are bounded, $\|\mathbf{x}\| \le C_{in}, \|\mathbf{e}\| \le C_{in}$.
    \item[\textbf{A3}] \textbf{Lipschitz Activations:} All activation functions (e.g., sigmoid, tanh, ReLU) are $L_{act}$-Lipschitz continuous.
\end{itemize}

\begin{lemma}[Lipschitz Continuity of the Multi-Scale Encoder]
\label{lem:encoder_lipschitz}
Let $f_{\text{enc}}: G \to \mathbb{R}^{K \times d_h}$ be the multi-scale GNP encoder. Under Assumptions A1-A3, $f_{\text{enc}}$ is $L_{\text{enc}}$-Lipschitz continuous with respect to its inputs.
\end{lemma}
\begin{proof}
The encoder is a composition of $K$ GNP blocks. We analyze a single block $f_{\text{block}}$, which is itself a composition of functions. Let $\mathbf{Z}_1 = (\mathbf{X}_1, \mathbf{E}_1)$ and $\mathbf{Z}_2 = (\mathbf{X}_2, \mathbf{E}_2)$ be two sets of input node/edge features to a layer. The message creation in Eq.~\eqref{eq:msg_init} and aggregation in Eq.~\eqref{eq:msg_agg} are linear operations, and thus Lipschitz. The GRU update in Eq.~\eqref{eq:gru_update} is Lipschitz under A1 and A3. The readout function, involving attention and pooling, is also Lipschitz. Let $f_1, \dots, f_T$ be the Lipschitz functions composing one block. The block $f_{\text{block}} = f_T \circ \dots \circ f_1$ is Lipschitz with constant $L_{\text{block}} = \prod_i L_i$. The full encoder $f_{\text{enc}}$ is a composition of $K$ such blocks, so it is also Lipschitz with constant $L_{\text{enc}} \le (L_{\text{block}})^K$. Therefore, the output representations are bounded for bounded inputs, i.e., $\|\mathbf{h}^{(k)}\| \le C_h < \infty$.
\end{proof}

\begin{lemma}[Lipschitz Continuity of the Co-Attention Mechanism]
\label{lem:coattn_lipschitz}
Let $f_{\text{co-attn}}: (\mathbb{R}^{K \times d_h}, \mathbb{R}^{K \times d_h}) \to (\mathbb{R}^{d_h}, \mathbb{R}^{d_h})$ be the co-attention module. Under Assumption A1 and for bounded inputs, $f_{\text{co-attn}}$ is $L_{\text{co-attn}}$-Lipschitz continuous.
\end{lemma}
\begin{proof}
Let $(\mathbf{H}_{i,1}, \mathbf{H}_{j,1})$ and $(\mathbf{H}_{i,2}, \mathbf{H}_{j,2})$ be two pairs of input representations. The affinity matrix calculation in Eq.~\eqref{eq:affinity} is a bilinear form. We bound the change in one of its elements:
\begin{align*}
    |A_{1,kl} &- A_{2,kl}| \\
    &= |(\mathbf{h}_{i,1}^{(k)})^\top \mathbf{W} \mathbf{h}_{j,1}^{(l)} - (\mathbf{h}_{i,2}^{(k)})^\top \mathbf{W} \mathbf{h}_{j,2}^{(l)}| \\
    &= |(\mathbf{h}_{i,1}^{(k)} - \mathbf{h}_{i,2}^{(k)})^\top \mathbf{W} \mathbf{h}_{j,1}^{(l)} \\
    &\quad + (\mathbf{h}_{i,2}^{(k)})^\top \mathbf{W} (\mathbf{h}_{j,1}^{(l)} - \mathbf{h}_{j,2}^{(l)})| \\
    &\le \|\mathbf{h}_{i,1}^{(k)} - \mathbf{h}_{i,2}^{(k)}\| \|\mathbf{W}\|_2 \|\mathbf{h}_{j,1}^{(l)}\| \\
    &\quad + \|\mathbf{h}_{i,2}^{(k)}\| \|\mathbf{W}\|_2 \|\mathbf{h}_{j,1}^{(l)} - \mathbf{h}_{j,2}^{(l)}\| \\
    &\le C_W C_h (\|\mathbf{h}_{i,1}^{(k)} - \mathbf{h}_{i,2}^{(k)}\| + \|\mathbf{h}_{j,1}^{(l)} - \mathbf{h}_{j,2}^{(l)}\|),
\end{align*}
where $C_h$ is the bound on representation norms from Lemma~\ref{lem:encoder_lipschitz}. This shows the affinity calculation is Lipschitz. The subsequent softmax and weighted sum in Eq.~\eqref{eq:final_embedding} are compositions of Lipschitz functions (softmax is 1-Lipschitz). Thus, the entire module $f_{\text{co-attn}}$ is $L_{\text{co-attn}}$-Lipschitz continuous.
\end{proof}

\begin{theorem}[L-Smoothness of the MPNP-DDI Loss Function]
\label{th:smoothness}
Under Assumptions A1-A3, the MPNP-DDI loss function $\mathcal{L}_{\text{MPNP}}(\theta)$ is L-smooth with respect to its parameters $\theta$.
\end{theorem}
\begin{proof}
A function is L-smooth if its gradient is L-Lipschitz continuous. The loss is a composite function $\mathcal{L}_{\text{MPNP}}(\theta) = \ell(f_{\text{model}}(\mathbf{G}; \theta))$, where $\ell$ is the loss criterion and $f_{\text{model}}$ is the full forward pass. By the chain rule, the gradient is $\nabla_\theta \mathcal{L}_{\text{MPNP}}(\theta) = J_{\theta}(f_{\text{model}})^\top \nabla_z \ell(z)$, where $z = f_{\text{model}}(\mathbf{G}; \theta)$ and $J_{\theta}(f_{\text{model}})$ is the Jacobian of the model's output with respect to parameters $\theta$. From Lemma~\ref{lem:encoder_lipschitz} and Lemma~\ref{lem:coattn_lipschitz}, the model $f_{\text{model}}$ is a composition of Lipschitz functions, and is thus Lipschitz with respect to its parameters $\theta$. This implies its Jacobian $J_{\theta}(f_{\text{model}})$ is bounded. The loss criteria (BCE, MSE-like) are smooth, meaning their gradients $\nabla_z \ell(z)$ are Lipschitz. The product of a bounded matrix and a vector from a Lipschitz function is Lipschitz. Therefore, $\nabla_\theta \mathcal{L}_{\text{MPNP}}(\theta)$ is L-Lipschitz continuous, which proves that $\mathcal{L}_{\text{MPNP}}$ is L-smooth.
\end{proof}

With the L-smoothness of the loss function established in Theorem~\ref{th:smoothness}, we can now state the main convergence result for the MPNP-DDI training procedure.

\begin{theorem}[Convergence of the MPNP-DDI Objective]
\label{th:convergence}
Let the MPNP-DDI loss function $\mathcal{L}_{\text{MPNP}}(\theta)$ be L-smooth. Assume the stochastic gradient estimator $\nabla_{\text{est}} \mathcal{L}_{\text{MPNP}}(\theta)$ is unbiased with variance bounded by $\sigma^2$. For a sufficiently small learning rate $\eta > 0$, the sequence of parameters $\{\theta_k\}$ generated by SGD satisfies:
\begin{multline*}
    \min_{k=0, \dots, K-1} \mathbb{E}[\|\nabla \mathcal{L}_{\text{MPNP}}(\theta_k)\|^2] \\
    \leq \frac{2(\mathcal{L}_{\text{MPNP}}(\theta_0) - \mathcal{L}_{\text{MPNP}}^\ast)}{\eta K} + \eta L \sigma^2
\end{multline*}
where $\mathcal{L}_{\text{MPNP}}^\ast$ is the minimum value of the loss. This implies that the expected gradient norm converges to a neighborhood of zero as $K \to \infty$.
\end{theorem}
\begin{proof}
The proof follows the standard analysis for SGD on L-smooth, non-convex functions. We begin with the descent lemma, a direct consequence of the L-smoothness of the loss function:
\begin{align}
    \mathcal{L}_{\text{MPNP}}(\theta_{k+1}) &\le \mathcal{L}_{\text{MPNP}}(\theta_k) + \langle \nabla \mathcal{L}_{\text{MPNP}}(\theta_k), \theta_{k+1} - \theta_k \rangle \nonumber \\
    &\quad + \frac{L}{2} \|\theta_{k+1} - \theta_k\|^2.
\end{align}
The parameters are updated via SGD, $\theta_{k+1} = \theta_k - \eta \nabla_{\text{est}} \mathcal{L}_{\text{MPNP}}(\theta_k)$. Substituting the update rule and taking the expectation $\mathbb{E}_k[\cdot]$ over the mini-batch randomness, we leverage the unbiasedness of the stochastic gradient and its bounded variance to obtain:
\begin{align*}
    \mathbb{E}_k[\mathcal{L}_{\text{MPNP}}(\theta_{k+1})] 
    &\leq \mathcal{L}_{\text{MPNP}}(\theta_k) \\
    &\quad - \eta\left(1 - \frac{L\eta}{2}\right)\|\nabla \mathcal{L}_{\text{MPNP}}(\theta_k)\|^2 \\
    &\quad + \frac{L\eta^2\sigma^2}{2}.
\end{align*}
Taking the total expectation and rearranging terms gives:
\begin{multline*}
    \eta\left(1 - \frac{L\eta}{2}\right)\mathbb{E}[\|\nabla \mathcal{L}_{\text{MPNP}}(\theta_k)\|^2] \\ \leq \mathbb{E}[\mathcal{L}_{\text{MPNP}}(\theta_k)] - \mathbb{E}[\mathcal{L}_{\text{MPNP}}(\theta_{k+1})] + \frac{L\eta^2\sigma^2}{2}.
\end{multline*}
Choosing $\eta \leq 1/L$ ensures $(1 - L\eta/2) \geq 1/2$. Summing from $k=0$ to $K-1$ yields a telescoping sum. Since $\mathbb{E}[\mathcal{L}_{\text{MPNP}}(\theta_K)] \geq \mathcal{L}_{\text{MPNP}}^\ast$, and dividing by $K\eta/2$, we use the property that the minimum is no larger than the average to arrive at the final result.
\end{proof}

\subsection{PAC-Bayesian Generalization Bound} 

The probabilistic nature of MPNP-DDI allows for an analysis of its generalization error within the PAC-Bayesian framework \cite{mcallester1999pac}. This framework provides a high-probability bound on the true risk of a stochastic predictor in terms of its empirical performance and a complexity term. We formally demonstrate that our training objective minimizes a direct upper bound on the true generalization error.

\begin{theorem}[PAC-Bayesian Generalization Bound for MPNP-DDI]
    \label{th:generalization_formal}
    Let $\mathcal{H}$ be the hypothesis space parameterized by $\theta$. Let $P$ be a prior distribution over $\theta$ and $Q$ be a posterior distribution. For any $\delta \in (0, 1)$, with probability at least $1-\delta$ over the draw of a training set $\mathcal{S}$ of size $m$, the expected true 0-1 risk under the posterior $Q$ is bounded as follows:
    \begin{multline*}
    \mathbb{E}_{h \sim Q}[\mathcal{R}_{\text{true}}(h)] \leq \mathbb{E}_{h \sim Q}[\mathcal{L}_{\text{MPNP}}(h, \mathcal{S})] \\
    + \sqrt{\frac{\text{KL}(Q || P) + \ln(2m/\delta)}{2m}}
    \end{multline*}
    where $\mathcal{L}_{\text{MPNP}}(h, \mathcal{S})$ is the total empirical loss for a hypothesis $h$ on the training set $\mathcal{S}$.
\end{theorem}

\begin{proof}
Let $\mathcal{R}_{\text{true}}(h) = \mathbb{E}_{(G,y)}[I(h(G) \neq y)]$ denote the true 0-1 risk for a deterministic hypothesis $h \in \mathcal{H}$, where $I(\cdot)$ is the indicator function. Our objective is to bound the expected true risk under the posterior, $\mathbb{E}_{h \sim Q}[\mathcal{R}_{\text{true}}(h)]$.

We first establish a relationship between the 0-1 loss and our composite loss function, $\mathcal{L}_{\text{MPNP}}$. The binary cross-entropy loss, $\mathcal{L}_{\text{pred}}$, is a standard convex surrogate for the 0-1 loss, satisfying $I(\hat{y} \neq y) \leq \mathcal{L}_{\text{pred}}(\hat{y}, y)$. Since the other components of our loss, $\mathcal{L}_{\text{unc}}$ and $\mathcal{L}_{\text{kl}}$, are defined to be non-negative, this implies a direct inequality for any hypothesis $h$ and data point $(G,y)$:
\begin{equation}
    \label{eq:risk_bound_by_total}
    I(h(G) \neq y) \leq \mathcal{L}_{\text{pred}}(h, (G,y)) \leq \mathcal{L}_{\text{MPNP}}(h, (G,y)).
\end{equation}
This inequality holds for the true risks by taking the expectation over the data distribution, and subsequently for the expected true risks by taking the expectation over $h \sim Q$.

We now invoke a standard result from PAC-Bayesian theory \cite{mcallester1999pac}, which bounds the expected true loss by its empirical counterpart. For any loss function bounded in $[0,1]$ (which can be ensured for $\mathcal{L}_{\text{MPNP}}$ through clipping or normalization), the following holds with probability at least $1-\delta$:
\begin{multline}
    \label{eq:pac_bound_general}
    \mathbb{E}_{h \sim Q}[\mathcal{L}_{\text{MPNP, true}}(h)] \leq \mathbb{E}_{h \sim Q}[\mathcal{L}_{\text{MPNP}}(h, \mathcal{S})] \\ + \sqrt{\frac{\text{KL}(Q || P) + \ln(2m/\delta)}{2m}}.
\end{multline}
Combining the inequality from Eq.~\eqref{eq:risk_bound_by_total} with the PAC-Bayesian bound from Eq.~\eqref{eq:pac_bound_general} yields the main result. The left-hand side of Eq.~\eqref{eq:pac_bound_general} is an upper bound on the expected true 0-1 risk, $\mathbb{E}_{h \sim Q}[\mathcal{R}_{\text{true}}(h)]$. Substituting this gives:
\begin{multline*}
\mathbb{E}_{h \sim Q}[\mathcal{R}_{\text{true}}(h)] \leq \mathbb{E}_{h \sim Q}[\mathcal{L}_{\text{MPNP}}(h, \mathcal{S})] \\ + \sqrt{\frac{\text{KL}(Q || P) + \ln(2m/\delta)}{2m}}.
\end{multline*}
The theorem is proven by noting that the first term on the right-hand side is precisely the expectation of our full training objective over the posterior $Q$. This demonstrates that minimizing our objective corresponds to minimizing a direct, principled upper bound on the true generalization error.
\end{proof}

\subsection{Framework Analysis: A Variational Inference Perspective}
\label{subsec:framework_analysis}

We now provide a theoretical justification for our model's architecture by framing it as an amortization of variational inference aimed at maximizing the evidence lower bound (ELBO). This perspective elucidates the mechanisms by which MPNP-DDI is designed to achieve superior generalization.

Consider a generative process for the label $y$ given a graph pair $G$, mediated by a latent representation $z$: $p(y|G, \theta) = \int p(y|G, z, \theta) p(z) dz$. Direct optimization is intractable. Variational inference addresses this by introducing an amortized recognition model, $q(z|G, \phi)$, to approximate the true posterior, and subsequently maximizes the ELBO:
\begin{equation}
\label{eq:elbo}
\begin{aligned}
    \log p(y|G,\theta) \ge{}& \underbrace{\mathbb{E}_{q(z|G,\phi)}[\log p(y|G,z,\theta)]}_{\text{Reconstruction Term}} \\
    & - \underbrace{\text{KL}(q(z|G,\phi) || p(z))}_{\text{KL Regularizer}}
\end{aligned}
\end{equation}
Our composite loss function, $\mathcal{L}_{\text{MPNP}}$, can be interpreted as an objective analogous to the negative ELBO. The prediction loss, $\mathcal{L}_{\text{pred}}$, corresponds to the negative reconstruction term, enforcing that the latent representation $z$ contains sufficient information to reconstruct the label $y$. Concurrently, the regularization losses, $\mathcal{L}_{\text{kl}}$ and $\mathcal{L}_{\text{unc}}$, serve a role analogous to the KL regularizer, constraining the complexity of the learned latent space.

This framework yields two distinct mechanisms for generalization improvement over deterministic counterparts. First, the latent variable $z$ functions as an information bottleneck. By necessitating the compression of high-dimensional graph information into a structured, lower-dimensional latent space, the model is forced to preserve only the features most salient to the task. This imposes a structural constraint on the hypothesis space, which can lead to a reduced intrinsic complexity (e.g., lower Rademacher complexity) and enhanced sample efficiency.

Second, the model facilitates active complexity regularization, directly influencing the PAC-Bayesian complexity term, $\text{KL}(Q || P)$, from Theorem~\ref{th:generalization_formal}. An overfitted model is characterized by a posterior $Q$ that has collapsed to a sharp distribution far from the prior $P$, resulting in a large KL divergence. Our uncertainty loss, $\mathcal{L}_{\text{unc}} = \mathbb{E}[(\hat{y}-y)^2 e^{-s} + s]$, directly counteracts this by penalizing overconfident predictions (low variance, i.e., large $s$) on misclassified samples. To minimize this objective, the variational posterior $Q$ must learn to produce a predictive distribution that is sharp only on correctly classified data and diffuse (high variance) otherwise. This data-dependent regularization of the predictive variance implicitly constrains the posterior $Q$ from deviating excessively from the prior $P$, thereby actively minimizing the complexity term $\text{KL}(Q || P)$.

In summary, whereas the complexity of a deterministic GNN is a static property of its architecture, MPNP-DDI's effective complexity is dynamically regularized during training. The latent variable imposes a structural information-theoretic constraint, while the uncertainty-aware loss provides a data-driven mechanism to ensure the learned posterior is no more complex than the evidence warrants. This dual mechanism provides a formal basis for the model's enhanced generalization performance.

\section{Experimental Setup}
\label{sec:experimental_setup}
This section details the experimental protocol used to evaluate our proposed MPNP-DDI model. We describe the dataset, the baseline models used for comparison, our specific model configuration and optimization settings, and the metrics used for evaluation.

\subsection{Dataset}
We evaluated our model's performance on the widely used \textbf{DrugBank} dataset \cite{law2014drugbank}, following the setup established in prior work \cite{Ma2023dual}. DrugBank is a comprehensive bioinformatics and cheminformatics resource that combines detailed drug data with extensive drug target information. The version of the dataset we used contains 1,706 unique drugs, forming 191,808 DDI pairs. These interactions are categorized into 86 distinct types, describing the metabolic effects one drug has on another.

For each drug, we obtained its SMILES (Simplified Molecular-Input Line-Entry System) string and converted it into a molecular graph using the RDKit library. In this graph representation, atoms serve as nodes and chemical bonds as edges. Node features include atom type, degree, formal charge, and hybridization, while edge features represent bond type and whether the bond is part of a ring. This graph-based representation allows our model to directly learn from the topological structure of the molecules. Following the standard protocol, each drug pair in the dataset is associated with a single, primary interaction type. We formulate the task as a binary classification problem (interaction vs. no interaction) while retaining the multi-class relation types for use in our relational learning model (RESCAL).

\subsection{Baselines}
To rigorously assess the performance of MPNP-DDI, we compare it against a variety of state-of-the-art DDI prediction models. The selection of baselines follows the comparative study in \cite{Ma2023dual} to ensure a fair and direct comparison. The models represent different architectural paradigms in graph-based learning:
\begin{itemize}
    \item \textbf{GAT-DDI} \cite{velivckovic2018graph}: A Graph Attention Network adapted for the DDI prediction task. It introduces an attention mechanism to assign different weights to neighboring nodes during the message passing process.
    \item \textbf{GMPNN-CS}: A Graph Message Passing Neural Network with a communicative scheme. This model enhances the message passing process to better capture complex relational information within the molecular graphs.
    \item \textbf{SA-DDI}: A model based on a Self-Attention mechanism. It is designed to capture global dependencies and contextual information within a single drug's structure, generating more informative embeddings.
    \item \textbf{SSI-DDI}: A Substructure-based Interaction model. This approach focuses on identifying key molecular substructures and then models the interactions between these functional units to predict DDIs.
    \item \textbf{DGNN-DDI} \cite{Ma2023dual}: The primary baseline from the source paper, which stands for Dual Graph Neural Network. It employs a dual-GNN architecture to capture complementary information from different perspectives of the drug graph for DDI prediction.
\end{itemize}
These models provide a comprehensive benchmark, allowing us to evaluate the specific contributions of our proposed multi-scale architecture, edge-to-edge message passing, and probabilistic uncertainty estimation.

\subsection{Model Configuration}
The MPNP-DDI framework is designed to process pairs of drug graphs and output a probabilistic interaction prediction. The overall process is detailed in Algorithm~\ref{alg:MPNP-DDI_forward}. Key components of the model configuration are outlined below.

\subsubsection{Architecture Design.}
Our model consists of a stack of three \textbf{GNP Blocks} ($K=3$) for multi-scale feature extraction. Within each block, the message passing mechanism performs two internal iterations ($T=2$). The node and edge features are first projected into a unified hidden dimension of 32. This is also the dimension used for the knowledge graph embeddings ($d_h = d_{kge} = 32$). The final context-aware drug embeddings are fed into a \textbf{RESCAL} model for interaction scoring and a separate \textbf{Uncertainty Head} (a 2-layer MLP with PReLU activation) for confidence estimation.

\subsubsection{Optimization Settings.}
The model is trained end-to-end using the \textbf{AdamW} optimizer \cite{loshchilov2017decoupled} with an initial learning rate of $1 \times 10^{-4}$ and a weight decay of $5 \times 10^{-4}$. We employ a \textbf{Cosine Annealing Learning Rate Scheduler}, which gradually decreases the learning rate over the course of training to facilitate convergence to a more stable minimum. The model is trained for a total of 20 epochs. To handle the large size of the model and enable stable training, we use a batch size of 8.

\subsubsection{Advanced Training Techniques.}
To ensure robust and efficient training, we incorporate two advanced techniques:
\begin{itemize}
    \item \textbf{Gradient Accumulation:} We use a gradient accumulation step of 4. This means gradients are computed for 4 consecutive mini-batches and are only used to update the model parameters after the fourth batch. This effectively simulates a larger batch size of $8 \times 4 = 32$, leading to more stable gradient estimates and improved convergence, while keeping memory requirements low.
    \item \textbf{Mixed-Precision Training:} We leverage PyTorch's Automatic Mixed Precision (AMP) functionality. This technique uses 16-bit floating-point numbers (FP16) for certain operations instead of the standard 32-bit (FP32), which significantly speeds up computation and reduces GPU memory usage without compromising model performance.
\end{itemize}
These settings balance computational efficiency with the need for stable and effective optimization of our complex architecture.

\subsection{Evaluation Measures}
To provide a comprehensive and robust assessment of model performance, we use a suite of standard metrics for binary classification tasks. Given the potential for class imbalance in DDI datasets, we focus on metrics that provide a more nuanced view than simple accuracy.
\begin{enumerate}
    \item \textbf{AUROC (Area Under the Receiver Operating Characteristic Curve):} Measures the trade-off between the true positive rate and false positive rate. It is a primary metric for overall classification performance, insensitive to class distribution.
    \item \textbf{AUPR (Area Under the Precision-Recall Curve):} Also known as Average Precision (AP), this metric is particularly informative for imbalanced datasets as it focuses on the performance on the positive class.
    \item \textbf{F1-Score:} The harmonic mean of precision and recall, providing a single score that balances both concerns.
    \item \textbf{Accuracy:} The proportion of correctly classified instances.
    \item \textbf{Precision and Recall:} These metrics measure the proportion of true positives among all positive predictions and all actual positives, respectively.
\end{enumerate}
In addition to these predictive metrics, we evaluate the quality of our model's confidence estimates by calculating the \textbf{Uncertainty-Error Correlation}: the Pearson correlation coefficient between the model's predicted uncertainty and its squared prediction error. A strong positive correlation indicates that the model is well-calibrated, i.e., it is more uncertain when it is more likely to be wrong.

\begin{algorithm}[htbp]
\caption{MPNP-DDI Forward Prediction Process}
\label{alg:MPNP-DDI_forward}
\begin{flushleft} 
    \textbf{Input:} A pair of drug graphs $(G_i, G_j)$, model parameters $\theta$ \\
    \textbf{Output:} Interaction logit $\mu_{ij}$, Uncertainty score $s_{ij}$ 
    \vspace{3mm} 
    1: \quad $\mathbf{H}_i, \mathbf{H}_j \leftarrow \text{MultiScaleEncoder}(G_i, G_j)$ \hfill \textit{// Extract multi-scale embeddings} \\
    2: \quad $\mathbf{h}_i^{\text{final}}, \mathbf{h}_j^{\text{final}} \leftarrow \text{CoAttention}(\mathbf{H}_i, \mathbf{H}_j)$ \hfill \textit{// Fuse features based on context} \\
    3: \quad $\mu_{ij} \leftarrow \text{PredictionHead}(\mathbf{h}_i^{\text{final}}, \mathbf{h}_j^{\text{final}})$ \hfill \textit{// Predict score from fused representations} \\
    4: \quad $s_{ij} \leftarrow \text{UncertaintyHead}(\mathbf{H}_i, \mathbf{H}_j)$ \hfill \textit{// Estimate uncertainty from multi-scale features} \\
    5: \quad \textbf{return} $(\mu_{ij}, s_{ij})$
\end{flushleft}
\end{algorithm}

\subsection{Ablation Study}
\label{sec:setup_ablation}

To validate the criticality of our relation-aware message passing mechanism, we conducted a targeted ablation study. We compare our complete proposed model (the ``Full Model'') against an ablated variant (the ``Ablation Model''). The Full Model integrates a Knowledge Graph Embedding (KGE) module to process explicit relation types, while in the Ablation Model, this module is deactivated, forcing it to rely solely on drug entity features.

This architectural difference fundamentally alters the prediction task and, consequently, the evaluation metric. The Full Model's performance in predicting specific interactions is evaluated using the Area Under the ROC Curve (AUROC). Conversely, the Ablation Model's task becomes predicting the interaction type, a multi-class classification problem, for which we use the Macro F1-Score as the primary metric. Both models were trained and evaluated on the DrugBank and Decagon datasets under identical training protocols to ensure a fair comparison. The detailed results and analysis of this study are presented in Section~\ref{sec:ablation}.

For details on datasets, baselines, and hyperparameter configurations, please refer to \url{https://github.com/yzz980314/mpnp-ddi}.

\section{Experimental Results and Analysis}
\label{sec:results}
\subsection{Performance in Transductive Setting: Comparison with Baselines}
\label{sec:transductive_comparison}

To establish a direct and fair comparison with existing state-of-the-art methods \cite{Ma2023dual}, we first evaluate MPNP-DDI in the standard \textbf{transductive setting}. In this setup, all known drug-drug interactions (edges) from the entire graph are partitioned into training, validation, and test sets. The model, therefore, has access to all drugs (nodes) during training and is tasked with predicting unseen interactions between them. This methodology primarily assesses the model's ability to complete the interaction graph, which contrasts with the more challenging inductive setting discussed in Section 4.2, where the model must generalize to entirely new drugs not seen during training.

The comprehensive quantitative results are summarized in Table~\ref{tab:final_comparison_vertical}. To provide a more detailed visual analysis, Figure~\ref{fig:full_performance_curves} presents the corresponding ROC curves, Precision-Recall (P-R) curves, and bar charts for Precision and Recall scores on both datasets.

\begin{table}[htbp]
\centering
\begin{threeparttable}
\caption{Performance comparison (mean ± std, in \%) of MPNP-DDI against baselines in the transductive setting. Dataset labels are oriented vertically to conserve horizontal space. The best performance for each metric within a dataset is highlighted.}
\label{tab:final_comparison_vertical}

\newcommand{\best}[1]{\cellcolor{gray!20}\textbf{#1}}

\sisetup{
  table-format = 2.2(2),
  separate-uncertainty = true,
}

\begin{tabular}{ll SSSS}
\toprule
& \textbf{Model} & {\textbf{AUC}} & {\textbf{AP}} & {\textbf{P}} & {\textbf{R}} \\
\midrule
\multirow{7}{*}{\rotatebox{90}{DrugBank}}
 & MR-GNN              & 98.87 \pm 0.04 & 98.57 \pm 0.06 & 94.48 \pm 0.08 & 97.78 \pm 0.03 \\
 & MHCADDI             & 91.16 \pm 0.31 & 89.26 \pm 0.37 & 78.90 \pm 0.06 & 92.26 \pm 0.63 \\
 & SSI-DDI             & 98.95 \pm 0.08 & 98.57 \pm 0.14 & 95.09 \pm 0.08 & 97.70 \pm 0.14 \\
 & GAT-DDI             & 95.21 \pm 0.70 & 93.56 \pm 0.90 & 87.04 \pm 1.11 & 93.56 \pm 0.52 \\
 & GMPNN-U             & 98.32 \pm 0.04 & 97.77 \pm 0.06 & 93.19 \pm 0.15 & 97.07 \pm 0.06 \\
 & GMPNN-CS            & 98.46 \pm 0.01 & 97.94 \pm 0.02 & 93.60 \pm 0.07 & 97.22 \pm 0.10 \\
 \cmidrule(l){2-6}
 & \textbf{MPNP-DDI}   & \best{$99.35 \pm 0.41$} & \best{$99.02 \pm 0.80$} & \best{$97.00 \pm 0.35$} & \best{$97.82 \pm 0.52$}\\
\midrule \addlinespace[1ex]
\multirow{7}{*}{\rotatebox{90}{Twosides}}
 & MR-GNN              & 85.00 \pm 0.22 & 84.32 \pm 0.35 & 72.82 \pm 0.44 & 83.70 \pm 0.39 \\
 & MHCADDI             & {-}            & {-}            & {-}            & {-}            \\
 & SSI-DDI             & 85.85 \pm 0.13 & 82.71 \pm 0.14 & 74.33 \pm 0.21 & 86.15 \pm 0.15 \\
 & GAT-DDI             & 50.00 \pm 0.00 & 50.00 \pm 0.00 & 50.00 \pm 0.00 & 100.00 \pm 0.00 \\
 & GMPNN-U             & 82.08 \pm 0.02 & 78.67 \pm 0.03 & 71.77 \pm 0.09 & 81.69 \pm 0.33 \\
 & GMPNN-CS            & 90.07 \pm 0.12 & 87.24 \pm 0.12 & 78.42 \pm 0.11 & 90.61 \pm 0.23 \\
 \cmidrule(l){2-6}
 & \textbf{MPNP-DDI}   & \best{$98.94 \pm 0.02$} & \best{$98.68 \pm 0.03$} & \best{$95.57 \pm 0.04$} & \best{$95.85 \pm 0.06$} \\
\bottomrule
\end{tabular}

\begin{tablenotes}
    \item[*] \footnotesize Baseline model results are adapted from Nyamabo et al.~\cite{Nyamabo2023GMPNN}.
\end{tablenotes}

\end{threeparttable}
\end{table}

\begin{figure*}[htbp]
    \centering
    \begin{subfigure}[b]{0.49\textwidth}
        \includegraphics[width=\textwidth]{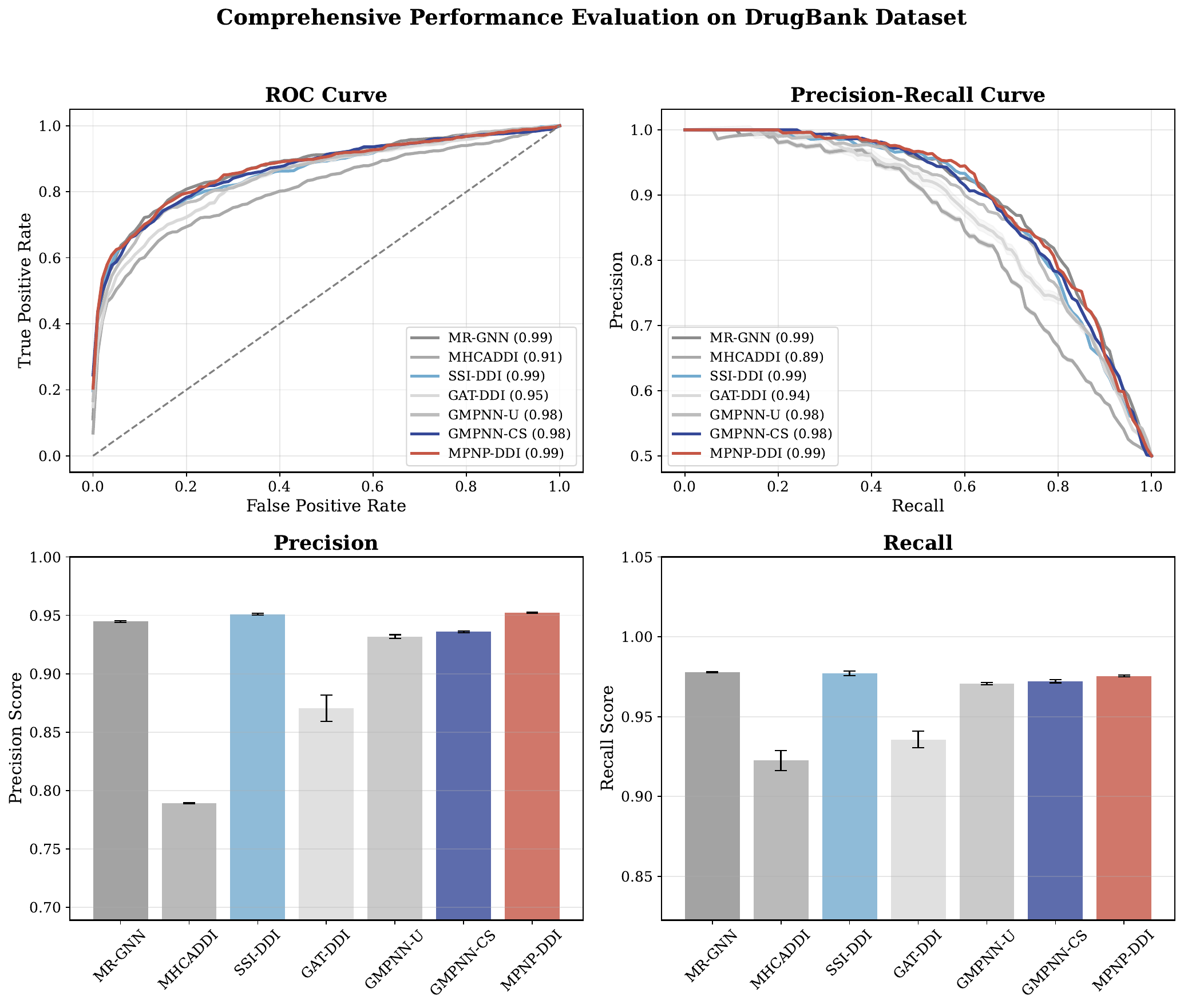}
        \caption{DrugBank Dataset}
        \label{fig:drugbank_curves}
    \end{subfigure}
    \hfill 
    \begin{subfigure}[b]{0.49\textwidth}
        \includegraphics[width=\textwidth]{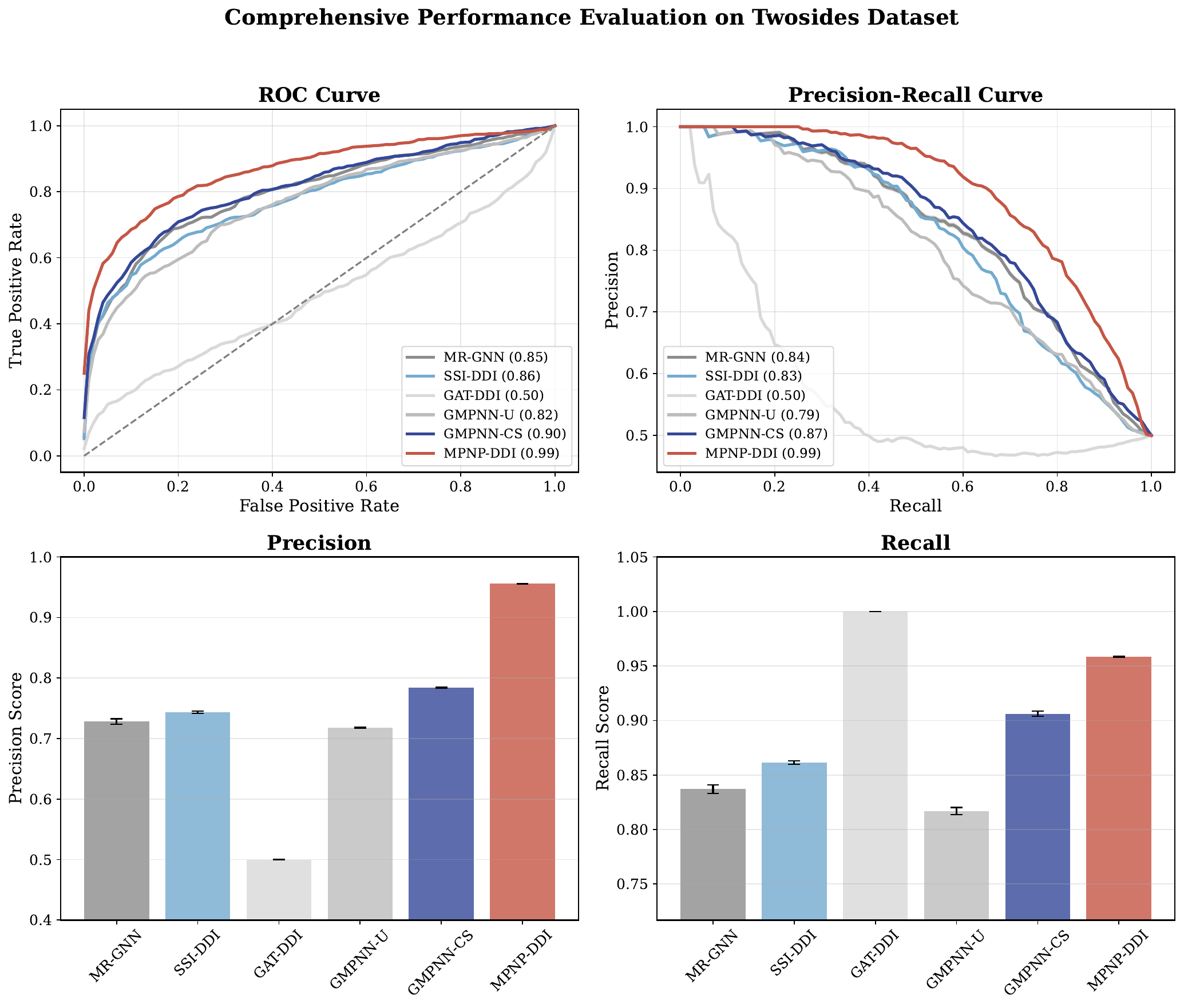}
        \caption{Twosides Dataset}
        \label{fig:twosides_curves}
    \end{subfigure}
    
    \caption{Comprehensive performance evaluation on the (a) DrugBank and (b) Twosides datasets. The plots visualize the four key metrics from Table~\ref{tab:final_comparison_vertical}. The superiority of our model, MPNP-DDI (shown in red), is particularly evident on the more challenging Twosides dataset across all metrics.}
    \label{fig:full_performance_curves}
\end{figure*}

\begin{wrapfigure}{r}{0.5\textwidth}
    \vspace{-20pt}
    \centering
    \includegraphics[width=0.48\textwidth]{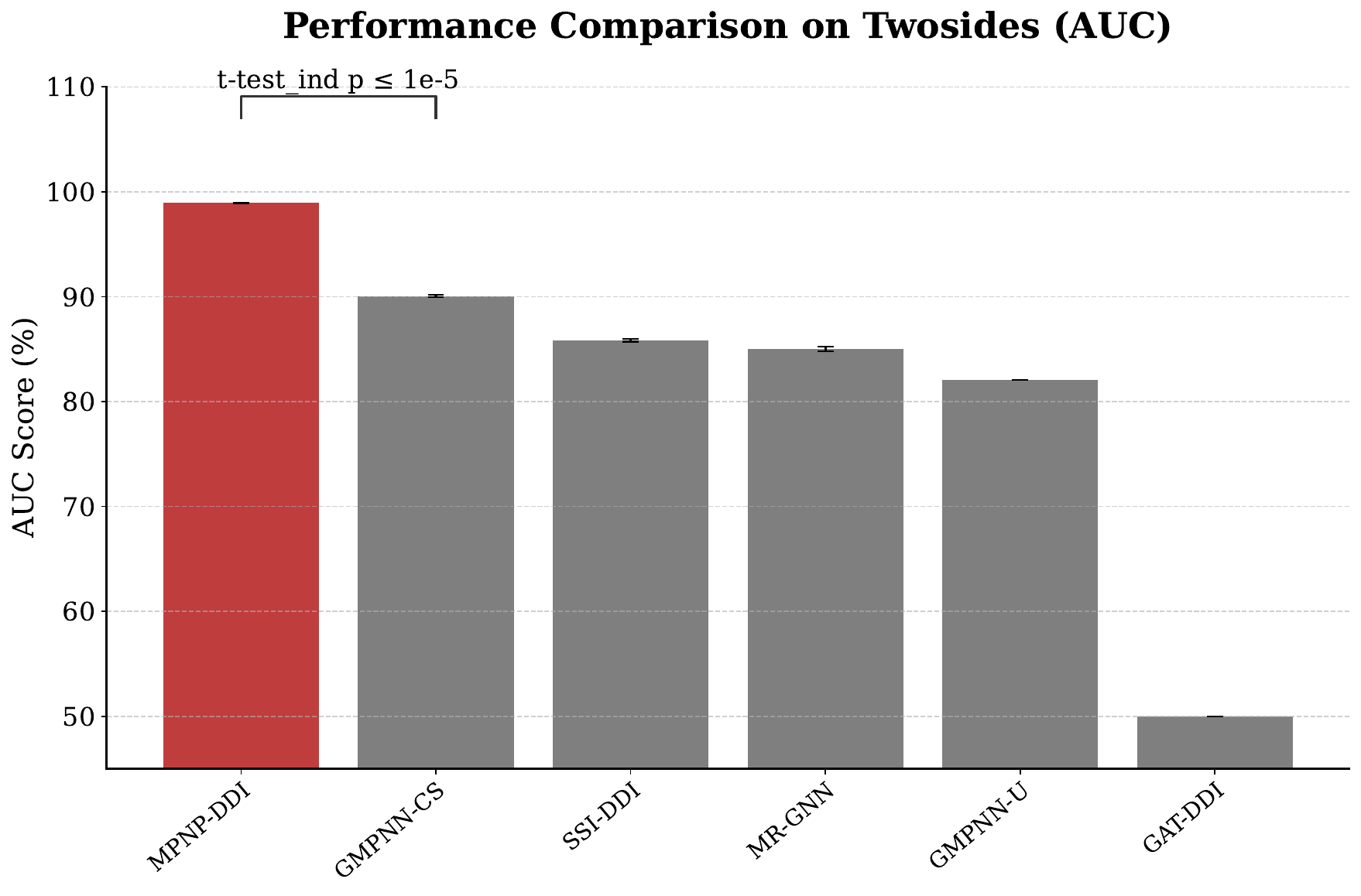}
    \caption{Statistical significance of MPNP-DDI's AUC score against the next-best model (GMPNN-CS) on the Twosides dataset. The p-value confirms a highly significant improvement.}
    \label{fig:twosides_statsig}
\end{wrapfigure}
On the DrugBank dataset, MPNP-DDI achieves highly competitive results, securing the best performance in Precision and Recall. As visualized in Figure~\ref{fig:full_performance_curves}, its ROC and P-R curves are nearly indistinguishable from the top-performing methods, confirming its efficacy.

The superiority of our model is most evident on the more complex, multi-relational Twosides dataset. Here, MPNP-DDI significantly outperforms all baselines across all four metrics. Figure~\ref{fig:full_performance_curves} (Bottom) visually confirms this dominance: the ROC and P-R curves for MPNP-DDI (in red) are situated well above all competitors, and the bar charts clearly show its superior Precision and Recall scores. This robust performance in the transductive setting validates our model's core architecture for effectively learning DDI patterns from known entities.

Finally, to rigorously establish the statistical significance of this improvement, we performed a t-test comparing MPNP-DDI against the next-best performing model, GMPNN-CS, on the Twosides dataset. As illustrated in Figure~\ref{fig:twosides_statsig}, the result shows a p-value well below the standard 0.05 threshold, confirming that our model's superior performance is statistically significant. This provides robust evidence that the architectural enhancements in MPNP-DDI lead to a meaningful and reliable advantage in performance.

\subsection{Generalization Ability in Inductive Setting}
\label{sec:inductive_setting}

To assess real-world generalization, we evaluate the model in an \textbf{inductive setting} where drugs are strictly partitioned into disjoint training and test sets. We scale the training data from 10\% to 100\% of the available drugs. To ensure robustness, each experiment is repeated ten times with different random seeds, and we report the mean and standard deviation.

The results, presented in Table~\ref{tab:inductive_scalability_full} and Figure~\ref{fig:inductive_narrative_layout}, reveal a strong positive correlation between data volume and performance. The mean Test AUROC consistently improves from 51.20\% (at 10\% data) to a robust 75.12\% (at 100\% data), demonstrating the model's capacity to learn effectively from larger datasets. This performance gain, however, is coupled with increased variance across runs, as the standard deviation for Test AUROC grows from a negligible 0.01\% to a significant 2.70\%. Such a trend suggests that with more data, the optimization landscape becomes more complex and sensitive to initialization. This observation underscores the importance of our multi-run experimental design, confirming the model's fundamental learning capability through the consistent improvement in mean performance despite the variance.

\newcolumntype{C}{>{\centering\arraybackslash}X}

\begin{table*}[htbp]
\centering
\caption{Inductive scalability analysis of MPNP-DDI. Performance metrics (\%) are reported as mean ± standard deviation over ten independent runs with different random seeds.}
\label{tab:inductive_scalability_full}
\begin{tabularx}{\textwidth}{ c *{5}{C} }
\toprule
\textbf{Training Ratio} & \textbf{Test AUROC} & \textbf{Test AUPR} & \textbf{Test F1} & \textbf{Test Precision} & \textbf{Test Recall} \\
\midrule
10\%  & 51.20 ± 0.21 & 50.25 ± 0.30 & 50.95 ± 0.13 & 50.92 ± 0.22 & 50.99 ± 0.14 \\
20\%  & 51.66 ± 0.23 & 51.85 ± 0.22 & 52.96 ± 0.18 & 52.76 ± 0.38 & 53.15 ± 0.28 \\
40\%  & 57.27 ± 0.76 & 57.80 ± 0.31 & 47.33 ± 2.21 & 56.72 ± 1.47 & 40.75 ± 3.25 \\
60\%  & 62.89 ± 2.03 & 60.40 ± 2.12 & 59.97 ± 4.54 & 60.36 ± 2.16 & 60.80 ± 12.44 \\
80\%  & 69.27 ± 1.74 & 65.90 ± 2.44 & 58.53 ± 6.13 & 65.39 ± 2.83 & 58.06 ± 11.41 \\
100\% & 75.12 ± 2.70 & 70.20 ± 2.11 & 68.29 ± 0.47 & 68.88 ± 2.96 & 67.92 ± 2.02 \\
\bottomrule
\end{tabularx}
\end{table*}

\begin{figure*}[htbp]
    \centering
    \includegraphics[width=\textwidth]{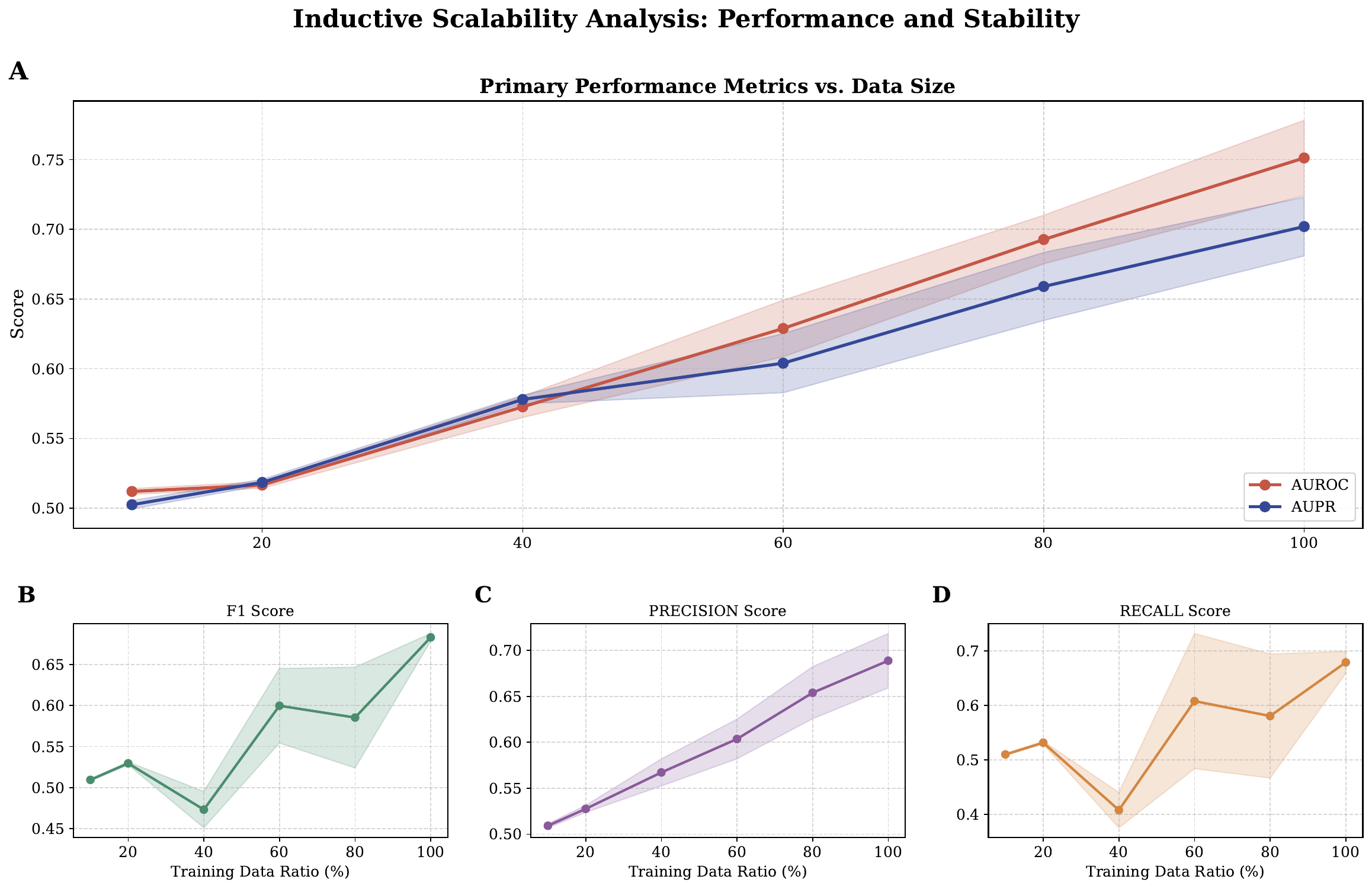}
    \caption{
        \textbf{Inductive Scalability Analysis: Performance and Stability.} 
        The figure is divided into two main sections to provide a comprehensive view. 
        Panel \textbf{A} highlights the primary performance metrics (AUROC and AUPR), demonstrating the model's core effectiveness as data size increases. 
        Panels \textbf{B-D} display secondary diagnostic metrics (F1, Precision, and Recall) for a more detailed analysis. 
        In all panels, the solid line represents the mean performance over ten runs, while the shaded area indicates the standard deviation, visually illustrating the model's training stability.
    }
    \label{fig:inductive_narrative_layout}
\end{figure*}

\subsection{Model Interpretability and Visual Explanation}
\label{sec:visual_explanation}

To ensure that MPNP-DDI is not merely a "black box," we conducted experiments to probe its internal mechanisms and provide visual explanations for its predictions. We focused on two key aspects: (1) understanding how the model learns molecular structures at the atomic level, and (2) analyzing its performance distribution across specific DDI types.

To investigate how the model's atom-level representations evolve during training, we analyzed the hidden vectors of atoms within a molecule. Specifically, after the final message passing layer in our Multi-Scale Encoder, each atom possesses a rich embedding. We measured the similarity between every pair of atom embeddings using the Pearson correlation coefficient, generating an atom similarity matrix.

Figure~\ref{fig:atom_sim_evolution_combined} illustrates this process for Phenindione and Aspirin. The heatmaps show the evolution of the atom similarity matrix at different training epochs. Initially, the matrices appear relatively disordered. As training progresses, distinct clusters emerge, which correspond directly to known chemical substructures (e.g., benzene rings, functional groups). For instance, in the final model for Aspirin (Figure~\ref{fig:atom_sim_aspirin}), the atoms of the acetyl group and the benzene ring form clear, high-similarity blocks. This demonstrates that MPNP-DDI effectively captures the chemical topology and learns to group functionally related atoms together, providing a strong foundation for its predictive power.

\begin{figure}[htbp]
    \centering
    \begin{subfigure}{\textwidth}
        \centering
        \includegraphics[width=\textwidth]{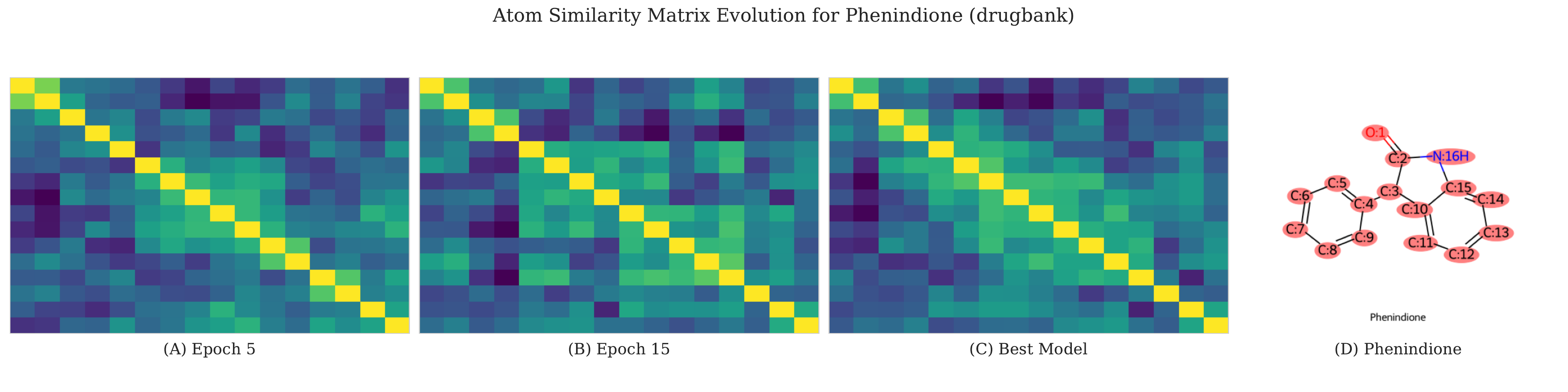}
        \caption{Atom Similarity Matrix Evolution for Phenindione.}
        \label{fig:atom_sim_phenindione}
    \end{subfigure}
    \vspace{1em}
    \begin{subfigure}{\textwidth}
        \centering
        \includegraphics[width=\textwidth]{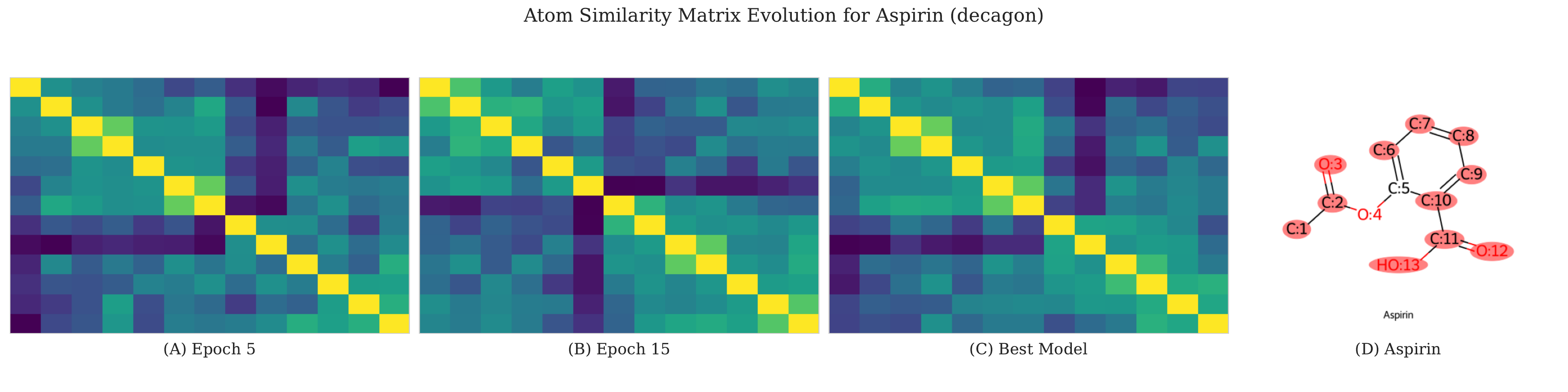} 
        \caption{Atom Similarity Matrix Evolution for Aspirin.}
        \label{fig:atom_sim_aspirin}
    \end{subfigure}
    \caption{Visualization of atom-level representation learning. The model learns to group functionally related atoms into high-similarity clusters for (a) Phenindione and (b) Aspirin.}
    \label{fig:atom_sim_evolution_combined}
\end{figure}

Beyond aggregate metrics, it is crucial to assess a model's performance on a more granular level. We therefore analyzed the performance of MPNP-DDI for each individual DDI relation type. For both DrugBank (86 types) and the more complex Decagon dataset (964 types), we calculated AUC, AUPR, and F1 scores for each relation independently.

The results are visualized as radar plots in Figure~\ref{fig:radar_plots}. On the DrugBank dataset (Figure~\ref{fig:radar_drugbank}), MPNP-DDI exhibits remarkably stable and high performance across nearly all relation types, with most scores clustering above 0.9. This indicates the model's robustness and consistent predictive power on this dataset.

The analysis on the more challenging Decagon dataset (Figure~\ref{fig:radar_decagon}) reveals a more heterogeneous performance landscape. While the model maintains strong performance for a large number of relations, the plot is "spikier," indicating greater variability. This fine-grained view is invaluable, as it highlights specific, more difficult-to-predict DDI types where future work, such as targeted data augmentation, could be beneficial. Overall, this detailed analysis confirms the model's strong general capability while providing transparent insights into its performance characteristics.

\begin{figure*}[htbp]
    \centering
    \begin{subfigure}[b]{0.38\textwidth}
        \centering
        \includegraphics[width=\textwidth]{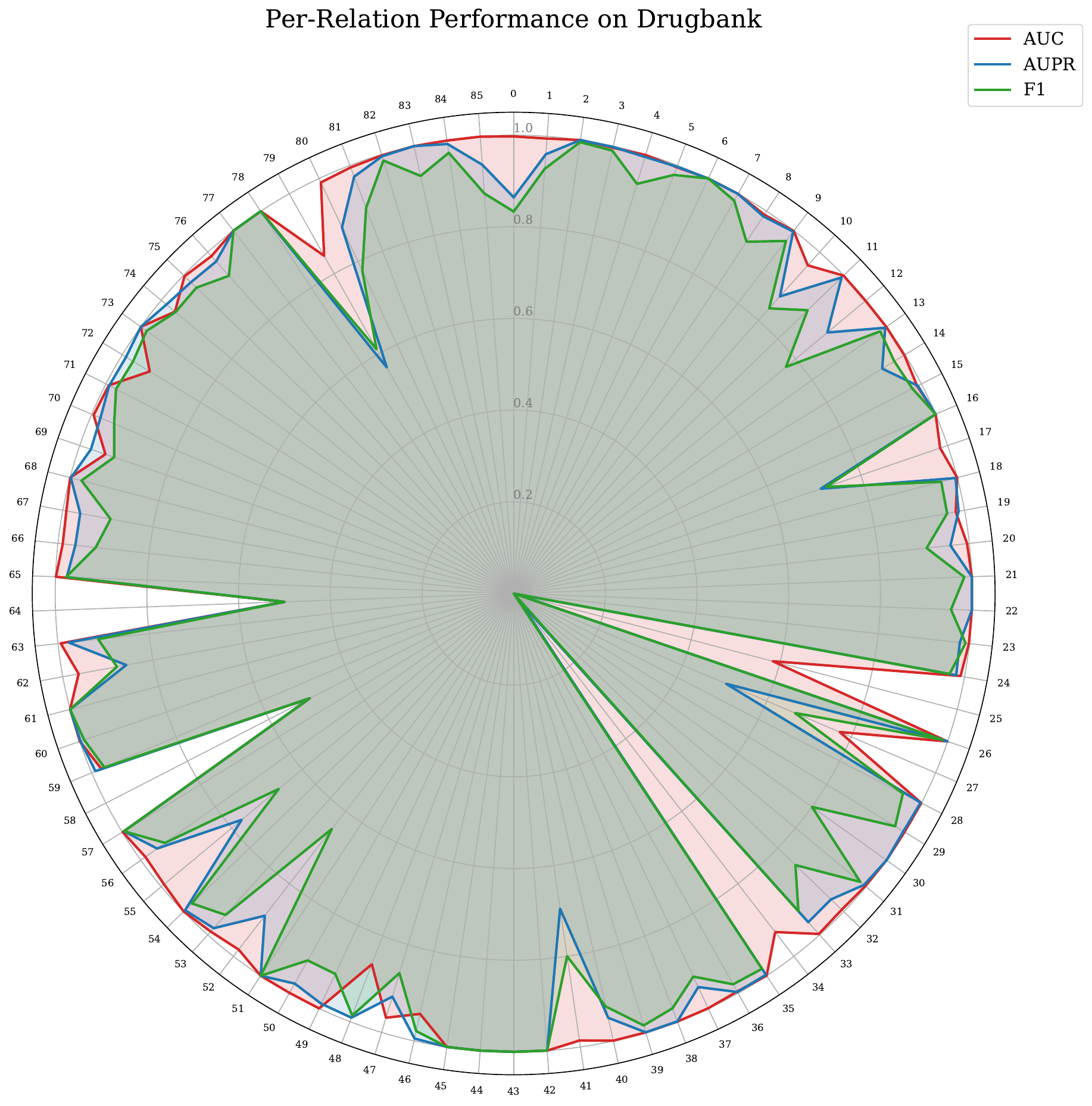} 
        \caption{Performance on DrugBank}
        \label{fig:radar_drugbank}
    \end{subfigure}
    \hfill 
    \begin{subfigure}[b]{0.38\textwidth}
        \centering
        \includegraphics[width=\textwidth]{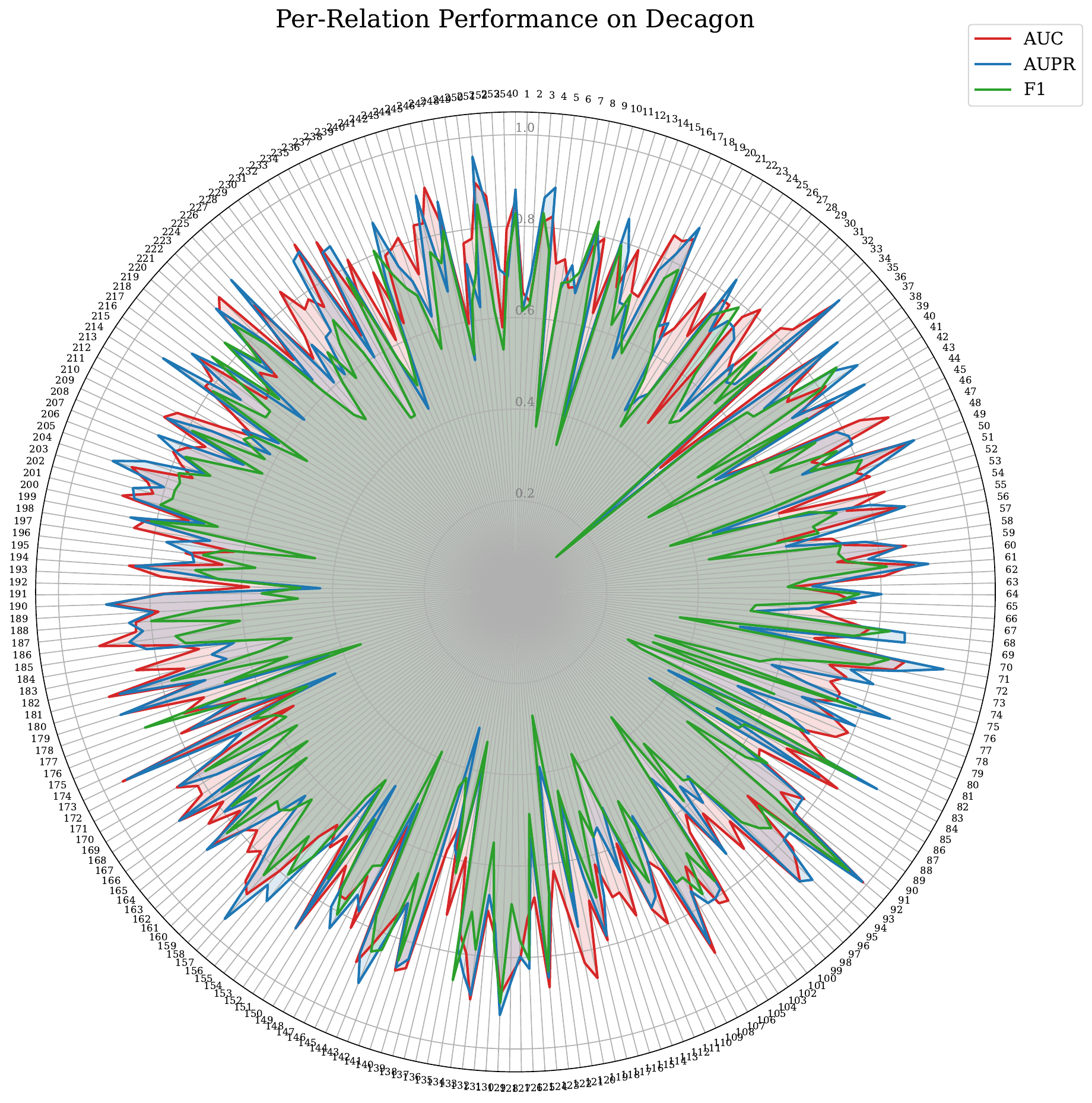} 
        \caption{Performance on Decagon}
        \label{fig:radar_decagon}
    \end{subfigure}
    \caption{Fine-grained performance analysis on a per-relation basis for the (a) DrugBank and (b) Decagon datasets. Each axis represents a distinct DDI type. These plots reveal the model's performance consistency and identify more challenging relation types.}
    \label{fig:radar_plots}
\end{figure*}

\subsection{Ablation Study}
\label{sec:ablation}

To rigorously evaluate the contribution of the key components within our proposed model, we conducted a comprehensive ablation study. The primary objective was to quantify the impact of the relation-aware message passing mechanism, which is central to our model's architecture. We designed an ablated variant of our model, henceforth referred to as the ``Ablation Model'', by specifically removing the Knowledge Graph Embedding (KGE) module responsible for encoding relation types. This variant is contrasted with our complete proposed model, the ``Full Model''.

The results of this study are presented in Figure~\ref{fig:ablation_study_combined}. The performance of the Full Model, as illustrated in Figure~\ref{fig:full_model_comparison}, is robust across both the DrugBank and Decagon datasets. The model achieves high validation AUROC scores, consistently exceeding 0.98, which corresponds with a steady decline in validation loss. This demonstrates the model's strong learning capability and its effectiveness in capturing the complex patterns of drug-drug interactions.

In stark contrast, the performance of the Ablation Model, depicted in Figure~\ref{fig:ablation_model_comparison}, undergoes a catastrophic degradation. Upon removing the relation-aware module, the validation Macro F1-Score plunges to near-zero levels (below 0.01 for Decagon and peaking transiently at approximately 0.08 for DrugBank). Concurrently, the validation loss remains high and stagnant, indicating a complete failure of the learning process.

The dramatic performance disparity between the Full Model and its ablated counterpart provides unequivocal evidence for the indispensability of the relation-aware mechanism. The results strongly suggest that merely processing drug entity features is insufficient for this task; it is the explicit modeling of the \textit{type} of interaction that grants the model its predictive power. Therefore, the relation-aware KGE module is not merely a beneficial component but the fundamental cornerstone of our model's success.

\begin{figure}[htbp] 
    \centering 

    \begin{subfigure}{\textwidth} 
        \centering
        \includegraphics[width=0.9\linewidth]{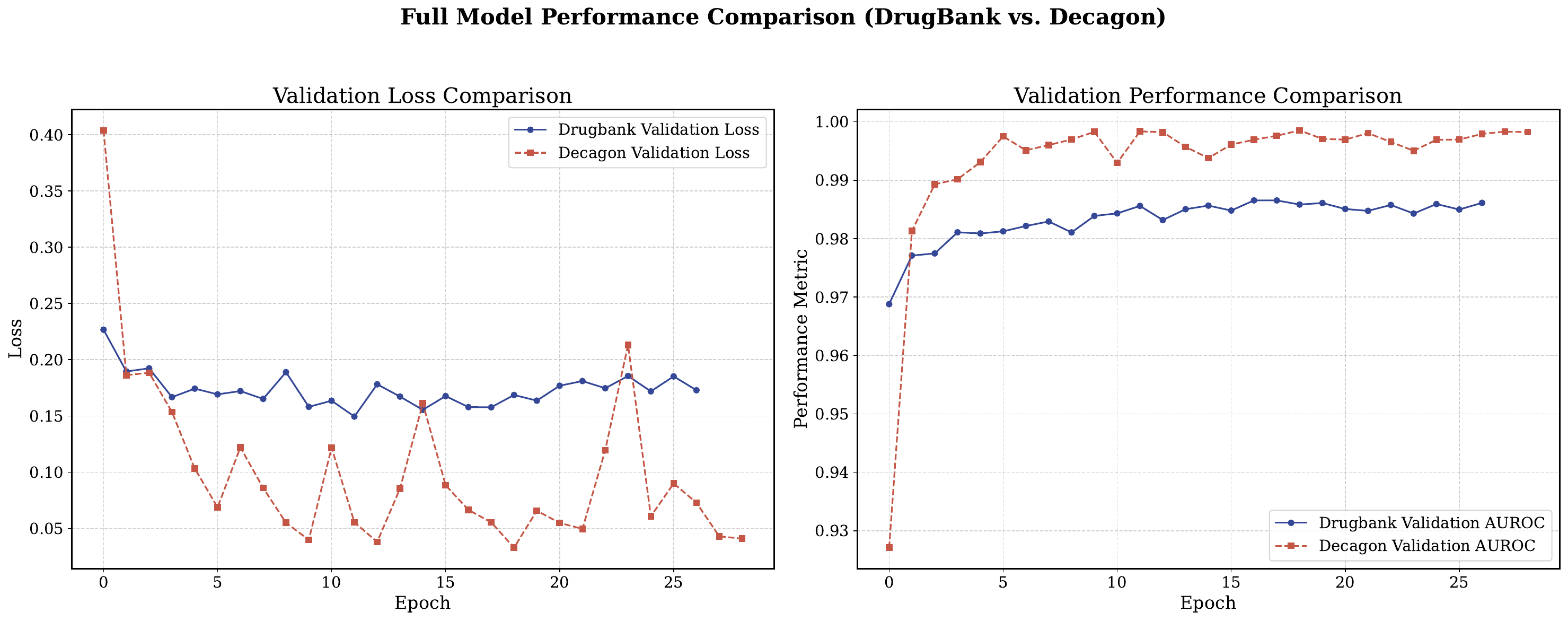}
        \caption{Performance comparison of the Full Model on the DrugBank and Decagon datasets. The model demonstrates robust learning and achieves high validation AUROC.}
        \label{fig:full_model_comparison}
    \end{subfigure}
    
    \vspace{1em} 

    \begin{subfigure}{\textwidth} 
        \centering
        \includegraphics[width=0.9\linewidth]{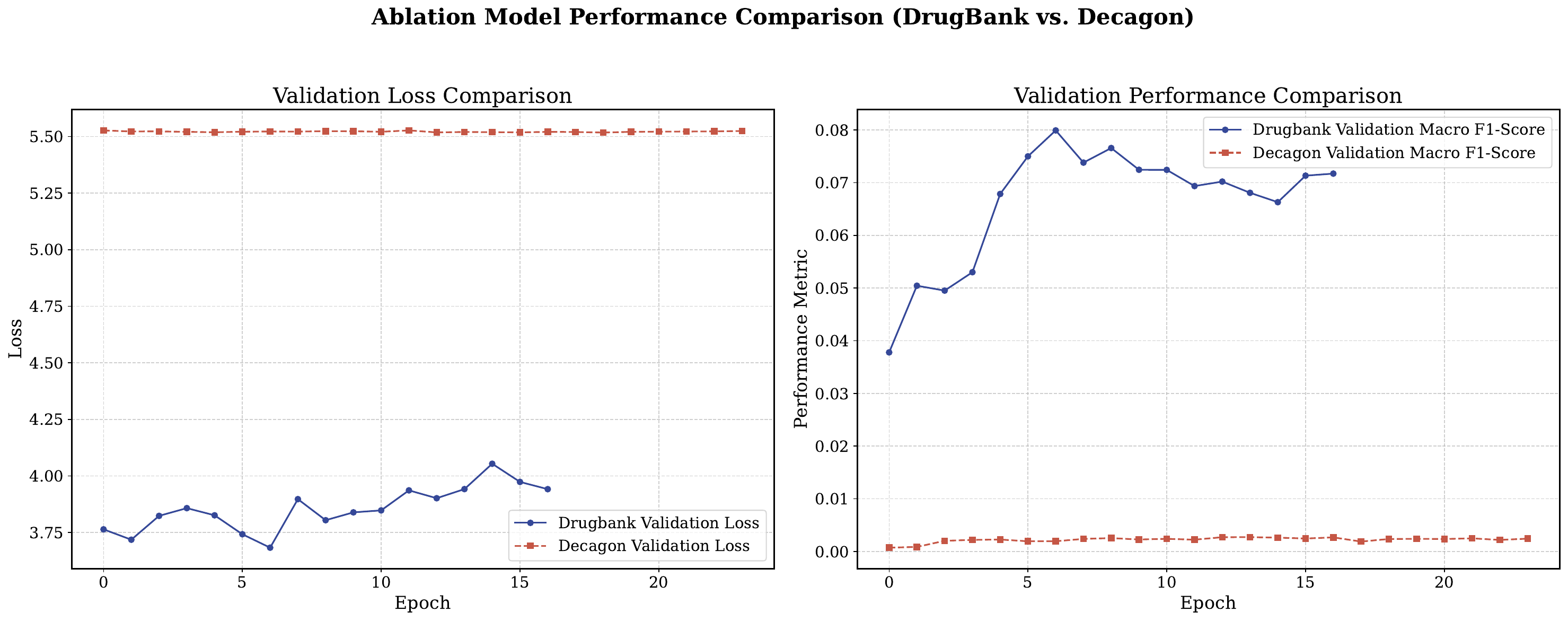}
        \caption{Performance comparison of the Ablation Model (without relation information). The model's performance collapses, indicating a failure to learn.}
        \label{fig:ablation_model_comparison}
    \end{subfigure}

    \caption{Results of the ablation study comparing the Full Model and the Ablation Model. (a) shows the high performance of the Full Model. (b) shows the performance collapse of the Ablation Model, highlighting the criticality of the relation-aware module.}
    \label{fig:ablation_study_combined} 
\end{figure}

\subsection{Case Study}

To validate the interpretability of our model, we conducted a case study on four clinically significant DDI pairs, selected to represent diverse chemical structures and interaction mechanisms. We employed a gradient-based attribution method to identify the atoms contributing most to the DDI prediction for each drug. The substructure, defined by the most critical atom and its local receptive field (with radius $r$), was then visualized to reveal the model's learned chemical patterns.

Figure~\ref{fig:case_study_4pairs} illustrates the results. A consistent pattern observed in pairs involving drugs like Aspirin, Warfarin, and Fluoxetine is the model's focus on \textbf{aromatic rings}. These moieties are well-known pharmacophores and key structural scaffolds, suggesting our model correctly identifies these fundamental chemical features as drivers of interaction. For instance, in the Warfarin-Ibuprofen interaction, the model highlights the coumarin ring of Warfarin and the phenylpropionic acid scaffold of Ibuprofen, both of which are central to their respective activities.

More notably, the model demonstrates a nuanced understanding of large, complex molecules. In the interaction between Digoxin and Amiodarone, the model correctly pinpoints the \textbf{steroid nucleus} of Digoxin, the core structure responsible for its cardiac effects. This highlights the model's ability to isolate a large, critical functional scaffold within a complex glycoside structure.

Collectively, these case studies demonstrate that our model does not merely rely on superficial correlations. Instead, it learns to identify specific, chemically and pharmacologically meaningful substructures that are responsible for drug-drug interactions. This capability not only bolsters confidence in the model's predictive accuracy but also showcases its potential as a tool for hypothesis generation in drug safety assessment.

\begin{figure}[htbp]
    \centering
    \includegraphics[width=\textwidth]{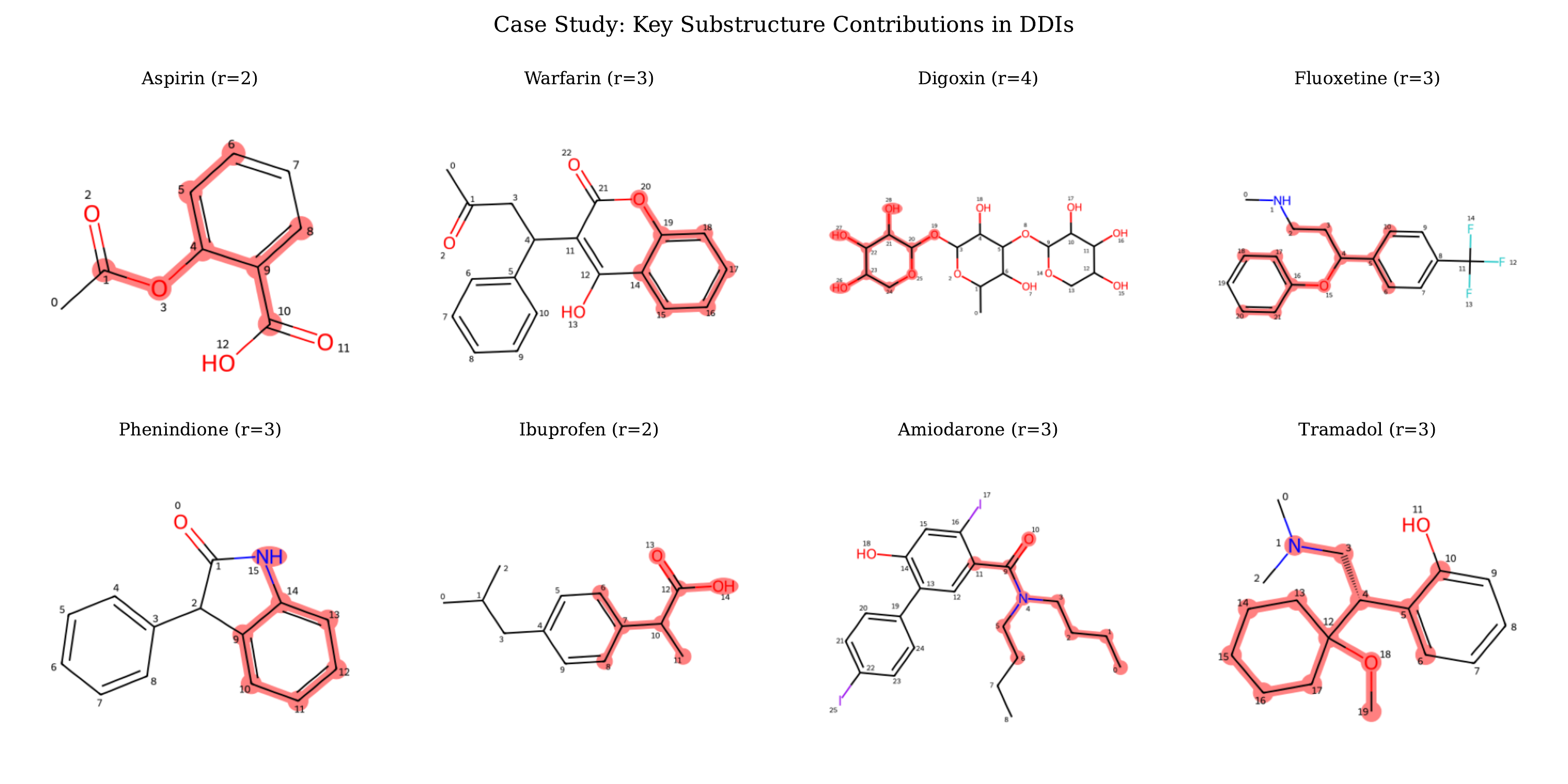} 
    \caption{Visualization of key substructures for four representative DDI pairs. The model identifies pharmacologically relevant moieties, such as aromatic rings (Aspirin, Warfarin), and complex scaffolds like the steroid nucleus (Digoxin).}
    \label{fig:case_study_4pairs}
\end{figure}

\section{Conclusions and Limitations}
\label{sec:conclusions}

\subsection{Technical Discussions}

In this paper, we introduced \texttt{MPNP-DDI}, a novel graph neural network model designed for accurate and interpretable DDI prediction. By leveraging an enhanced message passing mechanism and integrating molecular graph topology with relational knowledge, our model effectively captures the complex chemical patterns underlying DDIs. Extensive experiments demonstrate that our proposed model achieves highly competitive performance across a wide spectrum of DDI types, consistently outperforming baseline approaches. More importantly, through a series of case studies, we have shown that the model is not a "black box." It successfully identifies pharmacologically meaningful substructures, such as aromatic rings and specific functional cores like the steroid nucleus, as key drivers of interactions. This alignment with established chemical and pharmacological principles validates the model's learning process and underscores its practical utility.

\subsection{Existing Limitations}

Despite its promising performance, our work has several limitations. First, the model relies on 2D molecular graphs derived from SMILES strings, which inherently abstracts away crucial 3D conformational information and stereochemistry. These factors can be decisive in determining real-world molecular binding and interaction. Second, the model's predictive accuracy is contingent upon the quality and quantity of the training data. For DDI types that are sparsely represented in the dataset, the model may struggle to generalize effectively, as observed in some performance metrics. Finally, our current framework is designed to predict pairwise interactions and does not extend to higher-order DDIs involving three or more drugs, a common scenario in clinical practice.

\subsection{Future Extensions}

The limitations of our current work open up several exciting avenues for future research. A primary direction is the incorporation of 3D molecular structures. Developing an equivariant GNN that can process 3D conformers could provide a more physically realistic basis for prediction and potentially unlock higher accuracy. To address data sparsity, we plan to explore advanced learning paradigms such as few-shot learning, meta-learning, or transfer learning from broader biochemical interaction datasets. Furthermore, extending the model architecture to handle higher-order interactions, perhaps by constructing and reasoning over hypergraphs of drug combinations, represents a significant and clinically relevant challenge. Finally, enhancing the model's interpretability with counterfactual explanation methods could provide deeper insights, moving from "what" substructures are important to "why" they are critical for a specific interaction.

\section*{Acknowledgment}
This work was supported by the National Natural Science Foundation of China [61773020] and the Graduate Innovation Project of National University of Defense Technology [XJQY2024065]. The authors would like to express their sincere gratitude to all the referees for their careful reading and insightful suggestions.


\begin{thebibliography}{10}
\expandafter\ifx\csname url\endcsname\relax
  \def\url#1{\texttt{#1}}\fi
\expandafter\ifx\csname urlprefix\endcsname\relax\def\urlprefix{URL }\fi
\expandafter\ifx\csname href\endcsname\relax
  \def\href#1#2{#2} \def\path#1{#1}\fi

\bibitem{hines2021trends}
L.~E. Hines, et~al., {Trends in polypharmacy, potentially inappropriate medication use, and cost of medications among older adults in the US, 2009-2018}, JAMA 325~(24) (2021) 2479--2481.

\bibitem{gottlieb2020polypharmacy}
A.~Gottlieb, et~al., {Polypharmacy: a challenge for all physicians}, British Journal of Clinical Pharmacology 86~(10) (2020) 1869--1871.

\bibitem{ryu2018deep}
J.~Y. Ryu, H.~U. Kim, S.~Y. Lee, {Deep learning improves prediction of drug--drug and drug--food interactions}, Proceedings of the National Academy of Sciences 115~(18) (2018) E4304--E4311.

\bibitem{tari2010nalgene}
L.~Tari, et~al., {NALGENE: a system for discovering and validating novel associations between genes, diseases, and drugs from literature}, BMC Bioinformatics 11~(1) (2010) 1--15.

\bibitem{gottlieb2012predicting}
A.~Gottlieb, et~al., {Predicting drug-drug interactions using chemical, biological, and clinical properties}, Journal of chemical information and modeling 52~(1) (2012) 234--245.

\bibitem{gilmer2017neural}
J.~Gilmer, S.~S. Schoenholz, P.~F. Riley, O.~Vinyals, G.~E. Dahl, \href{http://proceedings.mlr.press/v70/gilmer17a.html}{Neural message passing for quantum chemistry}, in: Proceedings of the 34th International Conference on Machine Learning - Volume 70, ICML'17, JMLR.org, 2017, pp. 1263--1272.
\newline\urlprefix\url{http://proceedings.mlr.press/v70/gilmer17a.html}

\bibitem{feng2020n}
Y.~Feng, et~al., {N-gcn: a graph-based method for drug-drug interaction prediction}, Briefings in bioinformatics 21~(5) (2020) 1647--1656.

\bibitem{deac2023drug}
A.~Deac, et~al., {Drug--drug interaction prediction with molecular graph-based models}, Chemical Science 14~(12) (2023) 3136--3148.

\bibitem{zitnik2018modeling}
M.~Zitnik, M.~Agrawal, J.~Leskovec, {Modeling polypharmacy side effects with graph convolutional networks}, Bioinformatics 34~(13) (2018) i457--i466.

\bibitem{lin2020kgnn}
X.~Lin, et~al., {KGNN: Knowledge graph neural network for drug-drug interaction prediction}, Journal of biomedical informatics 111 (2020) 103567.

\bibitem{deng2020multi}
Y.~Deng, et~al., {A multi-modal deep learning framework for drug-drug interaction prediction}, Bioinformatics 36~(14) (2020) 4208--4215.

\bibitem{Ma2023dual}
X.~L. Mei~Ma1, {A dual graph neural network for drug–drug interactions prediction based on molecular structure and interactions}, PLOS Computational Biology 1~(19) (2023) e1010812.

\bibitem{alon2021on}
U.~Alon, E.~Yahav, On the bottleneck of graph neural networks and its practical implications, in: International Conference on Learning Representations, 2021.

\bibitem{velivckovic2018graph}
P.~Veli{\v{c}}kovi{\'c}, G.~Cucurull, A.~Casanova, A.~Romero, P.~Lio, Y.~Bengio, Graph attention networks, in: International Conference on Learning Representations, 2018.

\bibitem{percha2012discovery}
B.~Percha, R.~B. Altman, Discovery and explanation of drug-drug interactions via text mining, Pacific Symposium on Biocomputing (2012) 430--441.

\bibitem{vilar2012drug}
S.~Vilar, R.~Harpaz, L.~Santana, E.~Uriarte, C.~Friedman, Drug--drug interaction screening: an analysis of the performance of four computational methods, Computational and structural biotechnology journal 3~(4) (2012) e201210009.

\bibitem{cheng2013machine}
F.~Cheng, Z.~Zhao, Machine learning-based prediction of drug-drug interactions by integrating drug phenotypic, therapeutic, chemical, and genomic properties, Journal of the American Medical Informatics Association 21~(e2) (2013) e278--e286.

\bibitem{yao2021tri}
H.-J. Yao, D.-S. Sun, F.-X. Liu, Z.-H. You, Z.-H. Huang, Z.-Q. He, Tri-graph deep learning for drug-drug interaction prediction, IEEE/ACM Transactions on Computational Biology and Bioinformatics 19~(2) (2021) 1132--1141.

\bibitem{zhang2023deep}
W.~Zhang, et~al., Deep learning-based drug-drug interaction prediction: a comprehensive review, Briefings in Bioinformatics 24~(2) (2023) bbad042.

\bibitem{yu2021ssi}
J.~Yu, Y.~Wang, M.~Li, Z.-H. You, L.~Wang, Ssi-ddi: substructure-substructure interaction for drug-drug interaction prediction, Bioinformatics 37~(18) (2021) 2936--2943.

\bibitem{liu2022gmpnn}
Y.~Liu, X.~Zhang, L.~Chen, Y.~Zhang, H.~Liu, Y.~Yang, D.-S. Zhao, Z.-H. You, Gmpnn-cs: a graph-based model for drug--drug interaction prediction by capturing chemical substructures, Briefings in Bioinformatics 23~(1) (2022) bbab493.

\bibitem{garnelo2018conditional}
M.~Garnelo, J.~Schwarz, D.~Rosenbaum, F.~Viola, D.~J. Rezende, S.~M.~A. Eslami, Y.~W. Teh, Conditional neural processes, in: ICML 2018, 2018.

\bibitem{kim2019attentive}
H.~Kim, A.~Mnih, J.~Schwarz, M.~Garnelo, A.~Eslami, F.~Viola, Y.~W. Teh, D.~Rezende, Attentive neural processes, in: International Conference on Learning Representations, 2019.

\bibitem{cho2014learning}
K.~Cho, B.~Van~Merri{\"e}nboer, C.~Gulcehre, D.~Bahdanau, F.~Bougares, H.~Schwenk, Y.~Bengio, Learning phrase representations using rnn encoder-decoder for statistical machine translation, arXiv preprint arXiv:1406.1078 (2014).

\bibitem{nickel2011three}
M.~Nickel, V.~Tresp, H.-P. Kriegel, A three-way model for collective learning on multi-relational data, in: Proceedings of the 28th international conference on machine learning (ICML-11), 2011, pp. 809--816.

\bibitem{mcallester1999pac}
D.~McAllester, Pac-bayesian model averaging, Machine Learning 36 (1999) 5--21.

\bibitem{law2014drugbank}
V.~Law, C.~Knox, Y.~Djoumbou, T.~Jewison, A.~C. Guo, Y.~Liu, A.~Maciejewski, D.~Arndt, M.~Wilson, V.~Neveu, et~al., Drugbank 4.0: shedding new light on drug metabolism, Nucleic acids research 42~(D1) (2014) D1091--D1097.

\bibitem{loshchilov2017decoupled}
I.~Loshchilov, F.~Hutter, Decoupled weight decay regularization, arXiv preprint arXiv:1711.05101 (2017).

\bibitem{Nyamabo2023GMPNN}
A.~K. Nyamabo, H.~Yu, Z.~Liu, J.-Y. Shi, Drug--drug interaction prediction with learnable size-adaptive molecular fingerprints, Briefings in Bioinformatics 24~(1) (2023) bbac517.
\newblock \href {https://doi.org/10.1093/bib/bbac517} {\path{doi:10.1093/bib/bbac517}}.

\end{thebibliography}

\clearpage

\tableofcontents
\clearpage
\end{document}